\documentclass[acmsmall,screen,nonacm]{acmart}
\newtheorem{assumption}{Assumption}
\usepackage{multirow}
\usepackage{graphicx}
\usepackage{subfigure}

\AtBeginDocument{%
  }

\setcopyright{acmlicensed}
\copyrightyear{2018}
\acmYear{2018}
\acmDOI{XXXXXXX.XXXXXXX}

\acmJournal{TKDD}
\acmVolume{37}
\acmNumber{4}
\acmArticle{111}
\acmMonth{8}

\begin{document}

\title{Disentangled Graph Autoencoder for Treatment Effect Estimation}
\author{Di Fan}
\email{fandi@zju.edu.cn}
\orcid{0009-0001-6357-7849}
\affiliation{%
  \institution{School of Mathematical Sciences, Zhejiang
University}
  \streetaddress{866 Yuhangtang rd}
  \city{Xihu Qu}
  \state{Hangzhou}
  \country{China}
  \postcode{310012}
}

\author{Renlei Jiang}
\email{jiangrl@zju.edu.cn}
\orcid{}
\affiliation{%
  \institution{School of Mathematical Sciences, Zhejiang University}
  \streetaddress{866 Yuhangtang rd}
  \city{Xihu Qu}
  \state{Hangzhou}
  \country{China}
  \postcode{310012}
}

\author{Yunhao Wen}
\email{wenyunhao@petrochina.com.cn}
\orcid{}
\affiliation{%
  \institution{Petrochina Engineering and Planning Institute}
    \country{China}
}

\author{Chuanhou Gao}
\authornote{Corresponding author: Chuanhou Gao
}
\email{gaochou@zju.edu.cn}
\orcid{0000-0001-9030-2042}
\affiliation{%
  \institution{School of Mathematical Sciences, Zhejiang
University}
  \streetaddress{866 Yuhangtang rd}
  \city{Xihu Qu}
  \state{Hangzhou}
  \country{China}
  \postcode{310012}
}

\renewcommand{\shortauthors}{Fan et al.}

\begin{abstract}

Treatment effect estimation from observational data has attracted significant attention across various research fields. However, many widely used methods rely on the unconfoundedness assumption, which is often unrealistic due to the inability to observe all confounders, thereby overlooking the influence of latent confounders. To address this limitation, recent approaches have utilized auxiliary network information to infer latent confounders, relaxing this assumption. However, these methods often treat observed variables and networks as proxies only for latent confounders, which can result in inaccuracies when certain variables influence treatment without affecting outcomes, or vice versa. This conflation of distinct latent factors undermines the precision of treatment effect estimation. To overcome this challenge, we propose a novel disentangled variational graph autoencoder for treatment effect estimation on networked observational data. Our graph encoder disentangles latent factors into instrumental, confounding, adjustment, and noisy factors, while enforcing factor independence using the Hilbert-Schmidt Independence Criterion. Extensive experiments on multiple networked datasets demonstrate that our method outperforms state-of-the-art approaches.
\end{abstract}

\ccsdesc[500]{Computing methodologies~Causal reasoning and diagnostics}
\keywords{Causal inference, individual treatment effect, networked observational data, variational graph autoencoder, disentangled representation}

\received{20 February 2007}
\received[revised]{12 March 2009}
\received[accepted]{5 June 2009}

\maketitle
\section{Introduction}
Currently, research on causal effects between different variables has received increasing attention. Among these, learning individual treatment effects of a treatment on an outcome is a fundamental question encountered by numerous researchers, with applications spanning various domains, including education \cite{ding2010estimating}, public policy \cite{athey2016recursive}, 
economics \cite{zhang2021causal},
and healthcare \cite{shalit2017estimating}. For example, in the medical scenario, physicians seek to determine which treatment (such as which medication) is more beneficial for a patient's recovery \cite{wu2022learning}. This naturally raises a question: how can we accurately infer outcome if an instance were to receive an alternative treatment? This relates to the well-known problem of \textit{counterfactual outcome prediction} \cite{pearl2009causality}. By predicting counterfactual outcomes, we can accurately estimate each individual's treatment effect, known as \textit{individual treatment effect (ITE)} \cite{rubin2005causal,shalit2017estimating}, thereby assisting decision-making.

Randomized controlled trials (RCTs) are the gold standard method for learning causal effects \cite{pearl2009causality}.
In these trials, instances (experimental subjects) are randomly assigned to either the treatment or control group.
However, this is often costly, unethical, or even impractical \cite{yao2021survey}. 
Fortunately, the rapid increasing expansion of big data in many fields offers significant opportunities for causal inference research \cite{yao2021survey}, 
as these observational datasets are readily available and usually contain a large number of examples. Thus, we often concentrate on estimating treatment effects from observational data. 
In observational studies, treatment often depends on specific attributes of an instance, leading to selection bias \cite{imbens2015causal}. In the medical scenario above, socioeconomic status influences both medication choices and patient recovery. Higher socioeconomic status may increase access to expensive medications and positively impact health. Identifying and controlling for confounding factors (i.e., those affecting both treatment and outcome) is crucial for accurate predictions and presents the main challenge in learning ITE from observational data \cite{pearl2009causal,yao2021survey}. To address confounders, most existing methods rely on unconfoundedness assumption \cite{johansson2016learning,shalit2017estimating,yao2018representation}, 
meaning all confounders are measurable and embedded within observed features. However, this assumption is often unrealistic, as not all confounders can be measured. To relax this assumption, \citet{bennett2019policy} proposed using proxy variables to account for latent confounders.

Building on this idea, recent studies have incorporated network information, such as social networks, along with the observational features of units, to improve the accuracy of ITE estimation. In these datasets, instances are intrinsically linked by auxiliary network structures, often represented as an adjacency matrix $\mathbf{A}$ (as defined in Section \ref{section:Preliminaries}), such as user-related social networks, commonly referred to as \textit{networked observational data} \cite{guo2020learning,huang2023modeling}.  Unlike traditional observational data, utilizing the dependencies inherent in network data necessitates models that can effectively combine both network structure and observed variables. Several ITE estimation frameworks have been developed for networked observational data \cite{veitch2019using,guo2020learning,guo2021ignite,chu2021graph}, which typically leverage both network structures and noisy, measurable observed variables as proxy variables to learn and control for latent confounders.
For instance, socioeconomic status can be inferred from easily measurable variables (e.g., postal codes, annual income) combined with social network patterns (e.g., community affiliation). While these methods have shown empirical success, they assume that all proxy variables contribute to latent confounding factors. However, this assumption is incorrect, as some factors affect only treatment, others only outcomes, or may even be noise. As a result, these methods fail to accurately identify latent factors. In patient data, for example, age and socioeconomic status influence both treatment and outcome, acting as confounding factors (confounders); the attending physician affects only treatment, referred to as an instrumental factor; gene expression and air temperature affect only the outcome, referred to as adjustment factors; and information such as names and contact details serve as noise factors. Using all patient features and network information solely to learn latent confounding factors introduces new biases \cite{abadie2006large, haggstrom2018data}. Furthermore, these methods often attempt to balance confounding factor representations by minimizing the discrepancy between treatment and control groups, thus reducing selection bias \cite{guo2020learning,guo2021ignite}. However, in practice, feature distributions between groups often differ significantly, and strictly enforcing balance may undermine the validity of treatment effect estimation. As we will discuss later, a more reasonable approach is to balance the distribution of adjustment factor representations across groups. Overall, explicitly learning disentangled representations for these four types of latent factors is crucial for accurate ITE estimation in networked observational data.

To address the aforementioned challenges, we present a novel generative framework 
for estimating individual treatment effects on networked observational data. We name our model Treatment effect estimation on Networked observational data by Disentangled Variational Graph Autoencoder (\textit{TNDVGA}), which can effectively infer latent factors from observed variables and auxiliary network information using a graph autoencoder, while employing the Hilbert-Schmidt Independence Criterion (HSIC) independence constraint to disentangle these factors into four mutually exclusive sets, thereby improving treatment effect estimation.
TNDVGA builds on the Variational Graph Autoencoder (VGAE) \cite{kipf2016variational}, with specific modifications for ITE estimation. It introduces multiple encoder channels to capture different latent factors and incorporates specialized regularization terms, including factual outcome prediction, treatment prediction, and balanced representation, to enhance the disentangling process and improve overall accuracy.hese innovations enable TNDVGA to more effectively leverage network structure information to learn distinct latent factors, thereby enhancing ITE estimation. Our main contributions are:
\begin{itemize}
    \item We propose \textit{TNDVGA}, a framework for learning individual treatment effects from networked observational data. It simultaneously learns latent factor representations from observed variables and auxiliary network information, while disentangling these factors to improve the accuracy of treatment effect estimation.
    \item We incorporate the kernel-based Hilbert-Schmidt Independence Criterion to evaluate the dependence between different latent factor representations. This independence regularization is jointly optimized with other components of the model within a unified framework, facilitating the learning of independent, disentangled representations.
    \item We conduct extensive experiments to validate the effectiveness of our proposed framework TNDVGA. Results on multiple datasets indicate that TNDVGA achieves state-of-the-art performance, significantly outperforming baseline methods.
\end{itemize}

The rest of this article is organized as follows. The related work is reviewed in Section \ref{sec:Related work}. Section \ref{section:Preliminaries} introduces the technical preliminaries and problem statement. Section  \ref{sec:Methodology} describes the details of our proposed framework. We presents comprehensive experimental results of our model's performance on different datasets in Section \ref{sec:Experiments}. Finally, Section \ref{sec:conclusion and future work} concludes our work and suggests directions for future research.
\section{Related work}\label{sec:Related work}
Three aspects of related work are introduced in this section: (1) learning ITE from i.i.d observational data; (2) learning ITE from networked observational data; and (3) disentangled representations for treatment effect estimation.
\paragraph{Learning ITE from i.i.d observational data}
Due to the high cost and feasibility issues of randomized experiments, there has been growing interest in estimating individual-level causal effects from observational data, particularly with the rise of big data. BART \cite{chipman2010bart} utilizes dimensionally adaptive random basis functions for causal effect estimation, while Causal Forest \cite{wager2018estimation}, an extension of Breiman's random forest, estimates heterogeneous treatment effects nonparametrically. With the development of neural networks, techniques such as CFR \cite{shalit2017estimating} have emerged, leveraging representation learning to estimate individual treatment effects by mapping features to a latent space, capturing confounders, and minimizing prediction error. However, these methods rely on the strong ignorability assumption, which overlooks latent confounders and is often unrealistic in real-world observational studies. To relax this assumption, approaches like CEVAE \cite{louizos2017causal} estimate both the latent confounder space and causal effects simultaneously, while Deep-Treat \cite{atan2018deep} uses a bias-removing autoencoder and policy optimization network to derive balanced representations. Intact-VAE \cite{wu2021intact} treated the confounder as a latent variable in representation learning and proposed a VAE-based method to identify treatment effects. However, these methods do not account for the potential benefit of network structural information in observational data, which can improve ITE prediction accuracy.

\paragraph{Learning ITE from networked observational data} Recently, the emergence of networked observational data in various real-world tasks has prompted several studies to relax strong ignorability assumption by utilizing network information among different instances, where the network also serves as a proxy for unobserved confounders. NetDeconf \cite{guo2020learning} utilized network information and observed features to identify patterns of latent confounders, enabling the learning of valid individual causal effects from networked observational data. CONE \cite{guo2020counterfactual} further employed Graph Attention Networks (GAT) to integrate network information, thereby mitigating hidden confounding effects. IGNITE \cite{guo2021ignite} introduced a minimax game framework that simultaneously balances representations and predicts treatments to learn ITE from networked observational data. GIAL \cite{chu2021graph} leveraged network structure to capture additional information by identifying imbalances within the network for estimating treatment effects. \citet{thorat2023estimation} utilized network information to mitigate hidden selection bias in the estimation of ITE under networked observational studies with multiple treatments. 
However, these studies uniformly use all feature, including network information, to infer latent confounding factors without accounting for the presence of other factors, which may introduce estimation bias. Our work extends their approach by considering the distinct roles of various factors and employing disentangling techniques to effectively address this issue.

There is another line of work that investigates spillover effects or interference. In a network, the treatment administered to one instance may influence the outcomes of its neighbors. 
This phenomenon is known as spillover effects or interference \cite{arbour2016inferring,huang2023modeling}. 
Unlike these studies, we follow the assumption by \citet{guo2020learning} and \citet{veitch2019using} that conditioning on latent confounders separates each individual's treatment and outcome from those of others. We aim to exploit the network structure to learn different latent factors for accurately estimating ITE.

\paragraph{Disentangled representations for treatment effect estimation} 
From the perspective of causal representation learning, learning disentangled representations is one of the challenges in machine learning \cite{scholkopf2021toward}. 
Early methods primarily focused on variable decomposition \cite{kuang2017treatment,kuang2020data}, exploring treatment effect estimation by considering only adjustment variables and confounders as latent factors. This restricted approach resulted in imprecise confounder separation and hindered accurate estimation of individual treatment effects.
Subsequently, many methods focused on decomposing pre-treatment variables into instrumental variables, confounding variables, and adjustment factors \cite{hassanpour2019learning,wu2022learning,cheng2022learning}. 
More recently, methods based on Variational Autoencoders (VAE) \cite{kingma2013auto} have been proposed to address the disentanglement challenge in individual treatment effect estimation. TEDVAE \cite{zhang2021treatment} employed VAE to separate latent variables with a regularization term for treatment and outcome reconstruction. TVAE \cite{vowels2021targeted} integrated noise factors and introduced target learning regularization for ITE estimation. EDVAE \cite{liu2024edvae} tackled disentanglement from both data and model perspectives. VGANITE \cite{bao2022learning} combined VAE and Generative Adversarial Network (GAN) \cite{goodfellow2014generative} to disentangle factors into three distinct sets.
However, these studies primarily focus on estimating individual treatment effects from independent observational data. Given the importance of disentanglement in ITE estimation, it is essential to incorporate this approach when estimating ITE in network settings. Our model builds upon and extends these existing methods by enabling the use of network structure information while integrating disentanglement concepts to improve the accuracy of ITE estimation.

A method related to disentanglement approaches is instrumental variable (IV) regression, which addresses unobserved confounding through a two-stage process \cite{hartford2017deep,zhao2024networked}. 
However, it mainly focuses on the instrumental variable and neglects other latent factors, limiting its ability to capture the full complexity of causal relationships and reducing its effectiveness. In our work, we comprehensively consider different latent factors to achieve accurate ITE estimation.

\section{Preliminaries}\label{section:Preliminaries}
In this section, we first introduce the notations used in this article. We then outline the problem statement by providing the necessary technical preliminaries.

\textbf{Notations.} Throughout this work, we use unbold lowercase letters (e.g., $t$) to denote scalars, bold lowercase letters (e.g., $\mathbf{x}$) to represent vectors, and bold uppercase letters (e.g., $\mathbf{A}$) for matrices. The $(i, j)$-th entry of a matrix $\mathbf{A}$ is denoted by $\mathbf{A}_{ij}$.

\textbf{Networked observational data.} In the network observational data, we define the features (covariates) of $i$-th instance as $\mathbf{x}_i \in \mathbb{R}^k$, the treatment as $t_i$, and the outcome as $y_i \in \mathbb{R}$. 
We concentrate on cases where the treatment variable is binary, specifically $t_i \in \{0, 1\}$. We assume that all instances are connected through a network, represented by an adjacency matrix $\mathbf{A}$. The network is assumed to be undirected, with all edge weights equal\footnote{This work can be extended to weighted undirected networks and is also applicable to directed networks by utilizing specialized graph neural networks.}. Let $n$ denotes the number of instances, so that $\mathbf{A} \in \{0, 1\}^{n \times n}$. The entry $\mathbf{A}_{ij} = \mathbf{A}_{ji} = 1$ (or 0) indicates the presence (or absence) of an edge between the $i$-th and $j$-th instances. Therefore, the tuple $(\{\mathbf{x}_i, t_i, y_i\}_{i=1}^n, \mathbf{A})$ represents a network observational dataset. We denote $t_i = 1$ and $t_i = 0$ to represent whether the $i$-th instance is in the treatment or control group, respectively, without loss of generality.

\textbf{Individual treatment effects estimation.} Here, we present the background knowledge necessary for learning individual treatment effects. We make the assumption that for each pair of instance $i$ and treatment $t$, there exists a potential outcome $y_i^{t}$, representing the value that $y$ would take if treatment $t$ were applied to instance $i$ \cite{rubin1978bayesian}. Note that only one potential outcome is observable, while the unobserved outcome $y_i^{1 - t_i}$ is typically referred to as the counterfactual outcome. As a result, the observed outcome can be expressed as a function of the observed treatment and potential outcomes, given by $y_i = t_i y_i^1 + (1 - t_i) y_i^0$. Then the ITE for the instance $i$ in the context of networked observational data is defined as follows:
\begin{equation}
    \tau_i = \tau(\mathbf{x}_i, \mathbf{A}) = \mathbb{E}[y_i^1 \mid \mathbf{x}_i, \mathbf{A}] - \mathbb{E}[y_i^0 \mid \mathbf{x}_i, \mathbf{A}],
\end{equation}
which measures the difference between expected potential outcome under treatment and control for the instance $i$. Once ITE has been established, the average treatment effect (ATE) can then be estimated by averaging the ITE across all instances as $\text{ATE} = \frac{1}{n} \sum_{i=1}^{n} \tau_i$. Based on the aforementioned notations and definitions, we formally state the problem.
\begin{definition}[Learning ITEs from Networked Observational Data]
Given the networked observational data \((\{\mathbf{x}_i, t_i, y_i\}_{i=1}^n, \mathbf{A})\), our goal is to use the information from \((\mathbf{x}_i, t_i, y_i)\) and the network adjacency matrix \(\mathbf{A}\) to learn an estimate of the ITE $\tau_i$ for each instance \(i\).
\end{definition}

This paper is based on the following two fundamental assumptions necessary for estimating the individual treatment effect \cite{rosenbaum1983central}:
\begin{assumption}[Stable Unit Treatment Value Assumption (SUTVA)] 
The potential outcomes for one unit are not affected by the treatment assigned to other units.
\end{assumption}
\begin{assumption}[Overlap] 
Each unit has a nonzero probability of receiving either treatment or control given the observed variables, i.e., $0<P(t=1\mid\mathbf{x})<1$.
\end{assumption}

\textbf{Variational Graph 
Autoencoder \cite{kipf2016variational}.} Since our model builds upon VGAE, we provide the necessary background here.
VGAE is a widely used graph-based generative model designed for unsupervised learning of node representations. It follows an encoder-decoder architecture. The encoder (i.e. inference model), typically implemented as a Graph Neural Network (GNN) \cite{kipf2016semi}, learns a probabilistic distribution over the node embeddings by mapping the graph structure and node features to a latent space. Specifically, the encoder defines the variational posterior \( q(\mathbf{Z} | \mathbf{X}, \mathbf{A}) \), where  \( \mathbf{Z} \in \mathbb{R}^{n\times d}\) represents the matrix of latent node embeddings, $d$ is the latent dimension, $\mathbf{X} \in \mathbb{R}^{n\times k}$ is the feature matrix, and \( \mathbf{A} \) is the graph's adjacency matrix. The latent variable \( \mathbf{Z}  \) are drawn from this posterior distribution, parameterized by the encoder, which is widely recognized for effectively capturing graph dependencies.
The decoder (i.e. generative model) reconstructs the adjacency matrix from the learned latent representations. his reconstruction is based on the inner product of the latent representations for each pair of nodes and is expressed as:
\begin{gather}
p({\mathbf{A}}|\mathbf{Z})=\prod_{i=1}^N\prod_{i=1}^Np({\mathbf{A}}_{ij}|\mathbf{z}_i,z_j),\quad 
p({\mathbf{A}}_{ij}|\mathbf{z}_i,\mathbf{z}_j)=\hat{\mathbf{A}}_{ij}=\sigma(\mathbf{z}_i^T\mathbf{z}_j).
\end{gather}
where \( \sigma \) is the sigmoid function.

The training objective of VGAE is to maximize the Evidence Lower Bound (ELBO) to optimize the variational parameters. The ELBO consists of two terms: the reconstruction term and the regularization term. The goal is to maximize the log-likelihood of the graph's adjacency matrix \( \mathbf{A} \), while regularizing the latent space to follow a prior distribution, typically a standard normal distribution. The objective can be written as:
\begin{equation}
\mathcal{L} = \mathbb{E}_{q(\mathbf{Z} | \mathbf{X}, \mathbf{A})} \left[ \log p(\mathbf{A} | \mathbf{Z}) \right] -  D_{KL}(  q(\mathbf{Z} | \mathbf{X}, \mathbf{A}) || p(\mathbf{Z}) ),
\end{equation}
where the first term is the expected log-likelihood of the reconstructed adjacency matrix \( \mathbf{A} \) given the latent representations \( Z \),
and the second term is the Kullback-Leibler divergence \cite{Kullback1951oninformation} between the variational posterior \( q(\mathbf{Z} | \mathbf{X}, \mathbf{A}) \) and the prior \( p(\mathbf{Z}) \), typically chosen as a standard normal distribution.
Optimizing this objective allows the model to learn effective latent node embeddings for graph data.

\section{Methodology}\label{sec:Methodology}
In this section, we first present a theorem on the identifiability of the individual treatment effects. We then introduce our TNDVGA framework, which is designed to learn from networked observational data.
\subsection{Identifiability}
We introduce our model TNDVGA for estimating treatment effects, based on the assumption that the observed covariates $\mathbf{x}$ and the network patterns $\mathbf{A}$ can be regarded as generated from four distinct sets of latent factors $\mathbf{z} = (\mathbf{z}_t, \mathbf{z}_c, \mathbf{z}_y, \mathbf{z}_o)$. That is, we assume that $\mathbf{x}$ and $\mathbf{A}$ serve as proxy variables for all latent variables. In this context, $\mathbf{z}_t$ represents latent instrumental factors that influence the treatment but not the outcome, $\mathbf{z}_c$ includes latent confounding factors (latent confounders) that influence both the treatment and the outcome, $\mathbf{z}_y$ consists of latent adjustment factors that impact the outcome without affecting the treatment, and $\mathbf{z}_o$ refers to latent noise factors, which are covariates unrelated to either the treatment or the outcome. The inclusion of $\mathbf{z}_o$ allows for the effective removal of irrelevant variables when estimating $\mathbf{z}_t$, $\mathbf{z}_c$, and $\mathbf{z}_y$, thereby enhancing the accuracy of each latent factor in the learning process.
The proposed causal graph for ITE estimation is shown in Fig. \ref{fig:causal graph}. By explicitly modeling these four latent factors, it highlights that not all variables in the observed set serve as proxy variables for confounding factors, thereby effectively facilitates the learning of various types of unobserved factors.
\begin{figure}[t]
  \centering
  \includegraphics[width=0.26\linewidth]{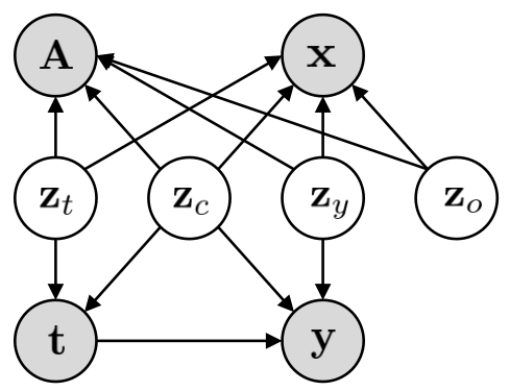}
  \caption{The causal diagram of the proposed TNDVGA. \( \mathbf{x} \) represents the observed variables, \( \mathbf{A} \) denotes the network structure, \( t \) is the treatment, \( y \) is the outcome, \( \mathbf{z}_t \) is latent instrument factors affecting only the treatment, \( \mathbf{z}_c \) is latent confounding factors, \( \mathbf{z}_y \) is latent adjustment factors affecting only the outcome, and \( \mathbf{z}_o \) is the latent noise factors unrelated to both treatment and outcome.
}
  \label{fig:causal graph}
  \Description{causal graph}
\end{figure}

Utilizing network observational data, we formulate and prove the identifiability of individual treatment effects in Theorem \ref{theorem:Identifiability of ITE}.
Assumption \ref{assum:unconfoundedness} is a relaxed version of the unconfoundedness assumption commonly used in causal inference \cite{rosenbaum1983central}, which typically assumes no latent confounders. In practice, depending on the specific context of the problem, this assumption is often relaxed \cite{bao2022learning}, such as in our case, to allow for the presence of latent confounders.

\begin{assumption}[Unconfoundedness (Ignorability)] \label{assum:unconfoundedness}
Treatment assignment is independent of the potential outcomes when conditioning on the latent confounding factors, i.e., $t \! \perp \!\!\! \perp(y^0,y^1)\mid \mathbf{z}_c$. 
\end{assumption}

\begin{theorem}[Identifiability of ITE]\label{theorem:Identifiability of ITE}
If we recover $p(\mathbf{z}_c,\mathbf{z}_y\mid\mathbf{x},\mathbf{A})$ and $p(y\mid t,\mathbf{z}_c,\mathbf{z}_y)$, then the proposed TNDVGA can recover the individual treatment effect from networked observational data.
\end{theorem}
\begin{proof}
According to the aforementioned assumptions and networked observational data,  the potential outcome distribution for any instance $\mathbf{x}$ can be calculated as follows:
\begin{equation}
\begin{split}
     &p(y^t\mid \mathbf{x},\mathbf{A})\\
    &\overset{(i)}{=}\int_{\{ \mathbf{z}_t,\mathbf{z}_c,\mathbf{z}_y,\mathbf{z}_o\}} p(y^t\mid \mathbf{z}_t,\mathbf{z}_c,\mathbf{z}_y,\mathbf{z}_o)p(\mathbf{z}_t,\mathbf{z}_c,\mathbf{z}_y,\mathbf{z}_o\mid\mathbf{x},\mathbf{A})d\mathbf{z}_td\mathbf{z}_cd\mathbf{z}_yd\mathbf{z}_o\\
    &\overset{(ii)}{=}\int_{\{ \mathbf{z}_t,\mathbf{z}_c,\mathbf{z}_y,\mathbf{z}_o\}}p(y^t\mid t,\mathbf{z}_t,\mathbf{z}_c,\mathbf{z}_y,\mathbf{z}_o)p(\mathbf{z}_t,\mathbf{z}_c,\mathbf{z}_y,\mathbf{z}_o\mid\mathbf{x},\mathbf{A})d\mathbf{z}_td\mathbf{z}_cd\mathbf{z}_yd\mathbf{z}_o\\
    &\overset{(iii)}{=}\int_{\{ \mathbf{z}_t,\mathbf{z}_c,\mathbf{z}_y,\mathbf{z}_o\}}p(y\mid t,\mathbf{z}_t,\mathbf{z}_c,\mathbf{z}_y,\mathbf{z}_o)p(\mathbf{z}_t,\mathbf{z}_c,\mathbf{z}_y,\mathbf{z}_o\mid\mathbf{x},\mathbf{A})d\mathbf{z}_td\mathbf{z}_cd\mathbf{z}_yd\mathbf{z}_o\\
    &\overset{(iv)}{=}\int_{\{ \mathbf{z}_c,\mathbf{z}_y\}}p(y\mid t,\mathbf{z}_c,\mathbf{z}_y)p(\mathbf{z}_c,\mathbf{z}_y\mid\mathbf{x},\mathbf{A})d\mathbf{z}_cd\mathbf{z}_y.
\end{split}
\end{equation}
The Equation (i) is a straightforward expectation over $p(\mathbf{z}_t,\mathbf{z}_c,\mathbf{z}_y,\mathbf{z}_o\mid\mathbf{x},\mathbf{A})$, Equation (ii) follows from Assumption \ref{assum:unconfoundedness} based on the conditional independence assumption $t \! \perp \!\!\! \perp(y^0,y^1)\mid \mathbf{z}_c$, Equation (iii) is derived from the commonly used consistency assumption \cite{imbens2015causal}, and Equation (iv) can be obtained from the Markov property $y \! \perp \!\!\! \perp \mathbf{z}_t,\mathbf{z}_o\mid t,\mathbf{z}_c,\mathbf{z}_y$. Thus, if we can model $p(\mathbf{z}_c,\mathbf{z}_y\mid\mathbf{x},\mathbf{A})$ and $p(y\mid t,\mathbf{z}_c,\mathbf{z}_y)$ correctly, then the ITE can be identified.
\end{proof}
Previous work by \citet{zhang2021treatment} has derived the proof for identifiability under the assumption of ignorability, based on the inference of relevant parent factors from proxy variables and/or other observed variables. We took inspiration from their concept. However, in contrast to them, our model includes latent noise factors, which improves the composition of the latent variables by bringing it closer to reality. Furthermore, additional network information can be used with the proxy variable $\mathbf{x}$. With these two modifications, we demonstrate the identifiability of ITE in Theorem \ref{theorem:Identifiability of ITE}. Theorem \ref{theorem:Identifiability of ITE} highlights the importance of distinguishing different latent factors and utilizing only the appropriate ones for treatment effect estimation on networked observational data.
\subsection{The proposed framework: TNDVGA}
An overview of the proposed framework, TNDVGA, is shown in Fig. \ref{fig:model structure}, which learns individual treatment effects through networked observational data. The proposed framework consists of the following important components: (1) Learning Disentangled Latent factors; (2) Predicting Potential Outcomes and Treatment Assignments; (3) Enforcing Independence of Latent factors. We will provide a detailed explanation of these three components in the following sections.
\subsubsection{Learning Disentangled Latent factors}
From the theoretical analysis in the previous section, we have seen that eliminating unnecessary factors is essential to effectively and accurately estimating the treatment effect. However, in practice, we do not know the mechanism of generating $\mathbf{x}$ and $\mathbf{A}$ from $\mathbf{z}$ and the mechanism of disentangling $\mathbf{z}$ into different disjoint sets. This requires us to propose a method that can learn to disentangle the latent factors $\mathbf{z}$ and estimate ITE through what the model has learned.
\begin{figure}[t]
  \centering
  \includegraphics[width=\linewidth]{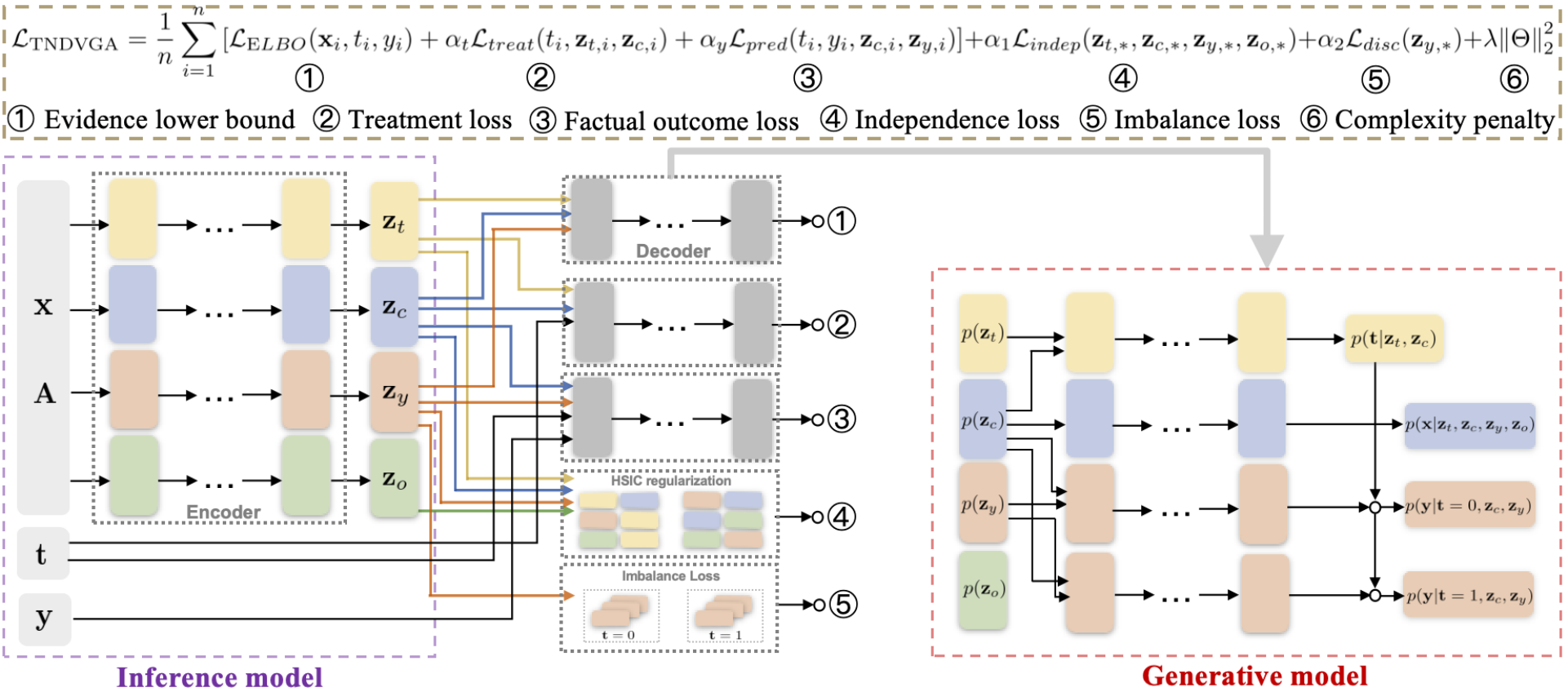}
  \caption{The overall architecture of TNDVGA consists of a generative network and an inference network for disentangling latent factors.}
  \label{fig:model structure}
  \Description{model structure}
\end{figure}
Therefore, we aim to infer the posterior distribution $p_\theta(\mathbf{z}\mid \mathbf{x},\mathbf{A})$ of latent factors $\mathbf{z}$ through the observed covariates $\mathbf{x}$ and network information $\mathbf{A}$, while disentangling $\mathbf{z}$ into latent instrumental factors $\mathbf{z}_t$, confounding factors $\mathbf{z}_c$, adjustment factors $\mathbf{z}_y$, noise factors $\mathbf{z}_o$. Since exact inference is intractable, we employ a variational inference framework to approximate the posterior distribution with a tractable distribution. 
Building upon the VGAE, we construct our model, which accounts for the graph structure inherent in the data. 
For each observed variable \( \mathbf{x} \), we define a multi-dimensional latent variable \( \mathbf{z} \). Furthermore, we incorporate an adjacency matrix \( \mathbf{A} \) into the encoder of the VGAE model, using it within the GNN to impose the structure of the posterior approximation
\( q_\phi(\mathbf{z} \mid \mathbf{x}, \mathbf{A}) \).
As shown in Fig. \ref{fig:model structure}, we use four separate encoders to approximate the variational posterior $q_{\phi_t}(\mathbf{z}_t\mid \mathbf{x},\mathbf{A})$, $q_{\phi_c}(\mathbf{z}_c\mid \mathbf{x},\mathbf{A})$, $q_{\phi_y}(\mathbf{z}_y\mid \mathbf{x},\mathbf{A})$, $q_{\phi_o}(\mathbf{z}_o\mid \mathbf{x},\mathbf{A})$, disentangling the latent variable $\mathbf{z}$ into $\mathbf{z}_t$, $\mathbf{z}_c$, and $\mathbf{z}_y$, $\mathbf{z}_o$, respectively\footnote{Our method does not employ $t$ and $y$ as inputs to the encoder as done in \cite{louizos2017causal}, because we assume that $t$ and $y$ are generated by the latent factors. So, the inference of latent factors relies solely on $\mathbf{x}$. For additional information, see \cite{zhang2021treatment}.}.
Then, these four latent factors are used by the decoder $p_\theta(\mathbf{x} \mid \mathbf{z}_t, \mathbf{z}_c, \mathbf{z}_y, \mathbf{z}_o)$ to reconstruct $\mathbf{x}$, $t$, and $y$\footnote{Note that, as shown in Figure \ref{fig:causal graph}, we obtain the independence property $\mathbf{x}  \! \perp \!\!\! \perp \mathbf{A} \mid \{\mathbf{z}_t, \mathbf{z}_c, \mathbf{z}_y, \mathbf{z}_o\}$. Thus, the 
decoder $p_\theta(\mathbf{x} \mid \mathbf{z}_t, \mathbf{z}_c, \mathbf{z}_y, \mathbf{z}_o, \mathbf{A}) = p_\theta(\mathbf{x} \mid \mathbf{z}_t, \mathbf{z}_c, \mathbf{z}_y, \mathbf{z}_o)$. The derivation for $t$ and $y$ is similar.}. Following the standard VGAE design, we select the prior distributions $p(\mathbf{z}_t)$, $p(\mathbf{z}_c)$,  $p(\mathbf{z}_y)$ and $p(\mathbf{z}_o)$ as factorized Gaussian distributions:
\begin{equation}
    \begin{split}
        p(\mathbf{z}_t)= \prod_{j=1}^{d_{\mathbf{z}_t}}\mathcal{N}(\{\mathbf{z}_{t}\}_j\mid 0,1);\quad p(\mathbf{z}_c)= \prod_{j=1}^{d_{\mathbf{z}_c}}\mathcal{N}(\{\mathbf{z}_{c}\}_j\mid 0,1);\\
        p(\mathbf{z}_y)= \prod_{j=1}^{d_{\mathbf{z}_y}}\mathcal{N}(\{\mathbf{z}_{y}\}_j\mid 0,1);\quad p(\mathbf{z}_o)= \prod_{j=1}^{d_{\mathbf{z}_o}}\mathcal{N}(\{\mathbf{z}_{o}\}_j\mid 0,1),
    \end{split}
\end{equation}
where $d_{\mathbf{z}_t}$, $d_{\mathbf{z}_c}$, $d_{\mathbf{z}_y}$, and $d_{\mathbf{z}_o}$ represent the dimensions of latent instrumental, confounding, adjustment, noise factors, respectively. And $\{\mathbf{z}_{t}\}_j$ denotes the $j$-th dimension of $\mathbf{z}_{t}$, and the same applies to $\mathbf{z}_{c}$, $\mathbf{z}_{y}$, and $\mathbf{z}_{o}$. 

The probabilistic representation of the generative model for $\mathbf{x}$, $t$, and $y$ is as follows:
\begin{align}
    p_{\theta_\mathbf{x}}(\mathbf{x} \mid \mathbf{z}_t, \mathbf{z}_c, \mathbf{z}_y, \mathbf{z}_o)&=\prod_{j=1}^{k}\mathcal{N}(\mu_j=f_{1j}(\mathbf{z}_{\{t,c,y,o\}}), \sigma_j^2=f_{2j}(\mathbf{z}_{\{t,c,y,o\}})),\\
    p_{\theta_t}(t \mid \mathbf{z}_t, \mathbf{z}_c)&= Bern(\sigma(f_3(\mathbf{z}_c,\mathbf{z}_t))),\label{eq:generative t}
\end{align}
\begin{equation}
\label{eq:generative y}
    \begin{split}
    p_{\theta_y}(y \mid t,\mathbf{z}_c, \mathbf{z}_y)&=\mathcal{N}(\mu=\hat{\mu},\sigma^2={\hat{\sigma}}^2),\\
        \hat{\mu}=tf_4(\mathbf{z}_c,\mathbf{z}_y)+(1&-t)f_5(\mathbf{z}_c,\mathbf{z}_y);{\hat{\sigma}}^2=tf_6(\mathbf{z}_c,\mathbf{z}_y)+(1-t)f_7(\mathbf{z}_c,\mathbf{z}_y),
    \end{split}
\end{equation}
where $f_1$ to $f_7$ are functions parameterized by fully connected neural networks, $\sigma(\cdot)$ represents the logistic function, and $Bern$ refers to the Bernoulli distribution. The distribution of $\mathbf{x}$ should be chosen based on the dataset, and in our case, we approximate it with a Gaussian distribution, as the data we use consists of continuous variables. Similarly, for the continuous outcome variable $y$, we also parameterize it as a Gaussian distribution, where the mean and variance are defined by two separate neural networks defining $p(y \mid t = 1, \mathbf{z}_c, \mathbf{z}_y)$ and $p(y \mid t = 0, \mathbf{z}_c, \mathbf{z}_y)$, following the two-headed approach proposed by \citet{shalit2017estimating}.

In the inference model, since we input the network information $\mathbf{A}$ into the encoder,
we utilize Graph Convolutional Networks (GCNs) \cite{kipf2016semi} as the encoder to obtain latent factor representations. GCN has been shown to effectively handle non-Euclidean data, such as graph-structured data, across diverse settings. To simplify notation, we describe the message propagation rule using a single GCN layer, as shown below:
\begin{equation}
 \mathbf{h}=GCN(\mathbf{x},\mathbf{A})={Relu}((\hat{\mathbf{A}}\mathbf{X})_\mathbf{x}\mathbf{W})={Relu}((\tilde{\mathbf{D}}^{\frac{1}{2}}\tilde{\mathbf{A}}\tilde{\mathbf{D}}^{-\frac{1}{2}}\mathbf{X})_\mathbf{x}\mathbf{W}).
 \label{eq:GCN}
\end{equation}
where $\mathbf{h} \in \mathbb{R}^d$ is the output vector of the GCN, $\mathbf{X} \in \mathbb{R}^{n \times k}$ is the feature matrix of the instances, $(\hat{\mathbf{A}}\mathbf{X})_\mathbf{x}$ represents the row of the matrix product corresponding to instance $\mathbf{x}$, $\tilde{\mathbf{A}}=\mathbf{A}+\mathbf{I}_N$, $\mathbf{I}_N$ is the identity matrix, $\tilde{\mathbf{D}}_{ii}=\sum_{j=1}^{N}\tilde{\mathbf{A}}_{ij}$, and $\mathbf{W} \in \mathbb{R}^{k \times d} $ denotes the parameters of the weight matrix. ${Relu}(\cdot)$ denotes the ReLU activation function. This leads to the following definition of the variational approximation of the posterior distribution for the latent factors:
\begin{equation}
    \begin{split}
        q_{\phi_t}(\mathbf{z}_t \mid \mathbf{x},\mathbf{A})=\mathcal{N}({\boldsymbol{\mu}}={\hat{\boldsymbol{\mu}}}_t,{\rm{diag}}(\boldsymbol{\sigma}^2)={\rm{diag}}({\hat{\boldsymbol{\sigma}}}_t^2)),\\
        q_{\phi_c}(\mathbf{z}_c \mid \mathbf{x},\mathbf{A})=\mathcal{N}({\boldsymbol{\mu}}={\hat{\boldsymbol{\mu}}}_c,{\rm{diag}}(\boldsymbol{\sigma}^2)={\rm{diag}}({\hat{\boldsymbol{\sigma}}}_c^2)),\\
        q_{\phi_y}(\mathbf{z}_y \mid \mathbf{x},\mathbf{A})=\mathcal{N}({\boldsymbol{\mu}}={\hat{\boldsymbol{\mu}}}_y,{\rm{diag}}(\boldsymbol{\sigma}^2)={\rm{diag}}({\hat{\boldsymbol{\sigma}}}_y^2)),\\
        q_{\phi_o}(\mathbf{z}_o \mid \mathbf{x},\mathbf{A})=\mathcal{N}({\boldsymbol{\mu}}={\hat{\boldsymbol{\mu}}}_o,{\rm{diag}}(\boldsymbol{\sigma}^2)={\rm{diag}}({\hat{\boldsymbol{\sigma}}}_o^2)),\\
    \end{split}
\end{equation}
where $\hat{\boldsymbol{\mu}}_t$, $\hat{\boldsymbol{\mu}}_c$, $\hat{\boldsymbol{\mu}}_y$, $\hat{\boldsymbol{\mu}}_o$ and ${\rm{diag}}(\hat{\boldsymbol{\sigma}}_t^2)$, ${\rm{diag}}(\hat{\boldsymbol{\sigma}}_c^2)$, ${\rm{diag}}(\hat{\boldsymbol{\sigma}}_y^2)$, ${\rm{diag}}(\hat{\boldsymbol{\sigma}}_o^2)$ are the means and covariance matrix of the Gaussian distributions, parameterized by the GCN as shown in Equation (\ref{eq:GCN}). Additionally, $\hat{\boldsymbol{\mu}}_t$ and ${\rm log}\,{\boldsymbol{\sigma}}_t^2$ are learned from two GCNs that share the training parameters of the first layer, and the same applies to the remaining three pairs.
\subsubsection{Predicting Potential Outcomes and Treatment Assignments}
The latent factors $\mathbf{z}_t$ and $\mathbf{z}_c$ are associated with the treatment $t$, whereas $\mathbf{z}_c$ and $\mathbf{z}_y$ are associated with the outcomes $y$, as illustrated in Fig. \ref{fig:causal graph}. To ensure that the treatment information is effectively captured by the union of $\mathbf{z}_t$ and $\mathbf{z}_c$, we add an auxiliary classifier to predict $t$ from the encoder's output, under the assumption that $\mathbf{z}_t$ and $\mathbf{z}_c$ can accurately predict $t$. Additionally, $y$ is predicted using two regression networks under different treatments to ensure that the outcome information is captured by the union of $\mathbf{z}_c$ and $\mathbf{z}_y$, based on the assumption that $\mathbf{z}_c$ and $\mathbf{z}_y$ can accurately predict $y$. Inspired by related approaches \cite{zhang2021treatment,liu2024edvae}, the classifier and regression networks are defined as follows: 
\begin{equation}
        q_{\eta_t}(t \mid \mathbf{z}_t, \mathbf{z}_c)= Bern(\sigma(h_1(\mathbf{z}_c,\mathbf{z}_t))),
\end{equation}
\begin{equation}
    \begin{split}
    q_{\eta_y}(y \mid t,\mathbf{z}_c, \mathbf{z}_y)&=\mathcal{N}(\mu=\hat{\mu},\sigma^2={\hat{\sigma}}^2),\\
        \hat{\mu}=th_2(\mathbf{z}_c,\mathbf{z}_y)+(1-t)h_3(\mathbf{z}_c,&\mathbf{z}_y),{\hat{\sigma}}^2=th_4(\mathbf{z}_c,\mathbf{z}_y)+(1-t)h_5(\mathbf{z}_c,\mathbf{z}_y),
    \end{split}
\end{equation}
where $h_1$ to $h_5$ are functions parameterized by fully connected neural networks, and the distribution settings are similar to those in Equations (\ref{eq:generative t}) and (\ref{eq:generative y}).
\subsubsection{Enforcing Independence of Latent factors}
Explicitly enhancing the independence of disentangled latent factors encourages the graph encoder to more effectively capture distinct and mutually independent information associated with each latent factor. In the following, we detail the regularization applied to enforce independence among the latent factors. 

The goal of our method is for the encoder to capture disentangled latent factors—namely, $\mathbf{z}_y$, $\mathbf{z}_c$, $\mathbf{z}_t$, and $\mathbf{z}_o$—that each contain exclusive information. This requires increasing the statistical independence between these latent factors to further strengthen disentanglement. Given the high dimensionality of the latent factors, using histogram-based measures is infeasible. Therefore, we use the Hilbert-Schmidt Independence Criterion (HSIC) \cite{gretton2005measuring} to promote sufficient independence among different latent factors. 

Specifically, let $\mathbf{z}_{t,*}$ represent the $d_{\mathbf{z}_t}$-dimensional random variable corresponding to the latent factor $\mathbf{z}_t$. Consider a measurable, positive definite kernel $\kappa_t$ defined over the domain of $\mathbf{z}_{t,*}$, with its associated Reproducing Kernel Hilbert Space (RKHS) denoted by $\mathcal{H}_t$. The mapping function $\psi_t(\cdot)$ transforms $\mathbf{z}_{t,*}$ into $\mathcal{H}_t$ according to the kernel $\kappa_t$. Similarly, for $\mathbf{z}_y$, $\mathbf{z}_c$, and $\mathbf{z}_o$, the same definitions apply. Given a pair of latent factors $\mathbf{z}_t$ and $\mathbf{z}_c$, where $\mathbf{z}_{t,*}$ and $\mathbf{z}_{c,*}$ are jointly sampled from the distribution $p(\mathbf{z}_{t,*}, \mathbf{z}_{c,*})$, the cross-covariance operator $\mathcal{C}_{\mathbf{z}_{t,*}, \mathbf{z}_{c,*}}$ in the RKHS of $\kappa_t$ and $\kappa_c$ is defined as:
\begin{equation}
    {\mathcal{C}}_{{\mathbf{z}}_{t,*}, \mathbf{z}_{c,*}}=\mathbb{E}_{p(\mathbf{z}_{t,*}, \mathbf{z}_{c,*})}
    \left[(\psi_t(\mathbf{z}_{t,*})-\boldsymbol{\mu}_{\mathbf{z}_{t,*}})^{\mathsf{T}}(\psi_t(\mathbf{z}_{c,*})-\boldsymbol{\mu}_{\mathbf{z}_{c,*}})\right],
\end{equation}
where $\boldsymbol{\mu}_{\mathbf{z}_{t,*}}=\mathbb{E}(\psi_t(\mathbf{z}_{t,*}))$, $\boldsymbol{\mu}_{\mathbf{z}_{c,*}}=\mathbb{E}(\psi_c(\mathbf{z}_{c,*}))$. Then, HSIC is defined as follows:
\begin{equation}
    {\rm{HSIC}}({\mathbf{z}}_{t,*},{\mathbf{z}}_{c,*}) :=  {\Vert {\mathcal{C}}_{{\mathbf{z}}_{t,*}, \mathbf{z}_{c,*}} \Vert}_{\rm HS}^2,
\end{equation}
where $\Vert \cdot \Vert$ is the Hilbert-Schmidt norm, which generalizes the Frobenius norm on matrices. It is known that for two random variables
$\mathbf{z}_{t,*}$ and $\mathbf{z}_{c,*}$ and characteristic
kernels $\kappa_{\mathbf{z}_{t,*}}$ and $\kappa_{\mathbf{z}_{c,*}}$, if $\mathbb{E}[\kappa_{\mathbf{z}_{t,*}}({\mathbf{z}_{t,*}},{\mathbf{z}_{t,*}})]<\infty, \mathbb{E}[\kappa_{\mathbf{z}_{c,*}}({\mathbf{z}_{c,*}},{\mathbf{z}_{c,*}})]<\infty$, then ${\rm HSIC}({\mathbf{z}}_{t,*},{\mathbf{z}}_{c,*})=0$ if and only if ${\mathbf{z}}_{t,*}\! \perp \!\!\! \perp {\mathbf{z}}_{c,*}$. In practice, we employ an unbiased estimator ${\rm HSIC}({\mathbf{z}}_{t,*},{\mathbf{z}}_{c,*})$  with $n$ samples \cite{song2012feature}, defined as:
\begin{equation}\label{eq:HSIC}
    {\rm{HSIC}}({\mathbf{z}}_{t,*},{\mathbf{z}}_{c,*})=\frac{1}{n(n-3)}\left[{\rm tr}(\tilde{\mathbf{U}}\tilde{\mathbf{V}}^{\mathsf{T}})+\frac{\mathbf{1}^{\mathsf{T}}\tilde{\mathbf{U}}\mathbf{1}\mathbf{1}^{\mathsf{T}}\tilde{\mathbf{V}}^{\mathsf{T}}\mathbf{1}}{(n-1)(n-2)}-\frac{2}{n-2}\mathbf{1}^{\mathsf{T}}\tilde{\mathbf{U}}\tilde{\mathbf{V}}^{\mathsf{T}}\mathbf{1}\right],
\end{equation}
where $\tilde{\mathbf{U}}$ and $\tilde{\mathbf{V}}$ denote the Grammer matrices with $\kappa_{\mathbf{z}_{t,*}}$ and $\kappa_{\mathbf{z}_{c,*}}$, respectively, with the diagonal elements set to zero. In our approach, we employ the radial basis function (RBF) kernel. The analysis for other pairs of latent factors follows similarly. 

The advantage of using Equation (\ref{eq:HSIC}) over other correlation measures to assess the dependence between different latent factors lies in its ability to capture complex, nonlinear dependencies by mapping the latent factors into the RKHS. The HSIC estimator we employ is unbiased, offering both effectiveness and computational efficiency. By bypassing the need for explicit estimation of the joint distribution of random variables, HSIC provides an efficient and reliable method for estimating dependencies across different representations.
\subsection{Loss Function of TNDVGA}
In this section, we design a loss function that combines all the key components of ITE estimation, thereby facilitating the end-to-end training of disentangled latent factor representations.
\subsubsection{Loss for Variational Inference} The encoder and decoder parameters can be learned by minimizing the negative evidence lower bound (ELBO), which serves as the primary loss function in standard VGAE. Here, $i$ denotes the $i$-th instance:
\begin{equation}
\begin{split}
    \mathcal{L}_{\rm ELBO}(\mathbf{x}_i, t_i, y_i)=\,&-\mathbb{E}_{q_{\phi_{t_i}}q_{\phi_{c_i}}q_{\phi_{y_i}}q_{\phi_{o_i}}}[{\rm log}\,p_{\theta_{\mathbf{x}_i}}({\mathbf{x}_i} \mid \mathbf{z}_{t,i}, \mathbf{z}_{c,i}, \mathbf{z}_{y,i}, \mathbf{z}_{o,i})+{\rm log}\,p_{\theta_{t_i}}(t_i \mid \mathbf{z}_{t,i}, \mathbf{z}_{c,i})\\
    &+\, {\rm log}\,p_{\theta_{y_i}}(y_i \mid t_i,\mathbf{z}_{c,i}, \mathbf{z}_{y,i})]-D_{KL}(q_{\phi_{t_i}}(\mathbf{z}_{t,i} \mid \mathbf{x}_i,\mathbf{A})\Vert p(\mathbf{z}_{t,i}))\\
    &-\,D_{KL}(q_{\phi_{c_i}}(\mathbf{z}_{c,i} \mid \mathbf{x}_i,\mathbf{A})\Vert p(\mathbf{z}_{c,i}))-D_{KL}(q_{\phi_{y_i}}(\mathbf{z}_{y,i} \mid \mathbf{x}_i,\mathbf{A})\Vert p(\mathbf{z}_{y,i}))\\
    &-\,D_{KL}(q_{\phi_{o_i}}(\mathbf{z}_{o,i} \mid \mathbf{x}_i,\mathbf{A})\Vert p(\mathbf{z}_{o,i})).
\end{split}
\end{equation}
Building on this loss, we further incorporate several useful regularization losses, as introduced below, into the variational lower bound to enhance the accuracy of ITE estimation.
\subsubsection{Loss for Potential Outcome Prediction and Treatment Assignment Prediction}
The factual loss function for predicting potential outcomes, along with the loss function for predicting treatment assignments, is defined as follows:
\begin{align}
    \mathcal{L}_{treat}(t_i, \mathbf{z}_{t,i}, \mathbf{z}_{c,i}) &= - \mathbb{E}_{q_{\phi_{t_i}}q_{\phi_{c_i}}}(q_{\eta_{t_i}}(t_i \mid \mathbf{z}_{t,i}, \mathbf{z}_{c,i})),\\
    \mathcal{L}_{pred}(t_i, y_i, \mathbf{z}_{c,i}, \mathbf{z}_{y,i}) &= -\mathbb{E}_{q_{\phi_{c_i}}q_{\phi_{y_i}}}(q_{\eta_{y_i}}(y_i \mid t_i,\mathbf{z}_{c,i}, \mathbf{z}_{y,i})).
\end{align}
\subsubsection{Loss for HSIC Independence Regularizer}
We apply pairwise independence constraints to the latent factors $z_t$, $z_c$, $z_y$, and $z_o$ in order to improve the statistical independence between disentangled representations. Thus, the HSIC regularizer $\mathcal{L}_{reg}$ is calculated as follows:
\begin{equation}
    \mathcal{L}_{indep}(\mathbf{z}_{t,*}, \mathbf{z}_{c,*}, \mathbf{z}_{y,*}, \mathbf{z}_{o,*})=\sum_{\stackrel{k,m\in \{t,c,y,o\}}{k\neq m}}{\rm HSIC}({\mathbf{z}}_{k,*}, {\mathbf{z}}_{m,*}).
\end{equation}
\subsubsection{Loss for Balanced Representation}
As shown in Fig. \ref{fig:causal graph}, we observe that $\mathbf{z}_y  \! \perp \!\!\! \perp t$, implying that $p(\mathbf{z}_y \mid t=0) = p(\mathbf{z}_y \mid t=1)$. Therefore, following the approach in \cite{hassanpour2019learning}, we aim for the learned $\mathbf{z}_y$ to exclude any confounding information, ensuring that all confounding factors are captured within $\mathbf{z}_c$. This is crucial for ensuring the accuracy of the treatment effect estimation. To quantify the discrepancy between the distributions of $\mathbf{z}_y$ for the treatment and control groups, we use the integral probability metric (IPM) \cite{sriperumbudur2012empirical,guo2020learning}. 
We define the balanced representation loss as $\mathcal{L}_{disc}$,
\begin{equation}
\label{eq:disc}
    \mathcal{L}_{disc}(\mathbf{z}_{y,*})=IPM(\{\mathbf{z}_{y,i}\}_{i:t_i=0}, \{\mathbf{z}_{y,i}\}_{i:t_i=1}).
\end{equation}
We utilize the Wasserstein-1 distance, defined as \citet{sriperumbudur2012empirical}, to calculate Equation (\ref{eq:disc}). Additionally, we employ the effective approximation algorithm proposed by \cite{cuturi2014fast} to calculate the Wasserstein-1 distance and associated gradients about the model parameters for training the TNDVGA.
\subsubsection{The Overall Objective Function} 
The following provides a summary of the overall objective function for TNDVGA:
\begin{equation}
    \begin{split}
    \mathcal{L}_{\rm{TNDVGA}}= \frac{1}{n}&\sum_{i=1}^n\left[\mathcal{L}_{\rm ELBO}(\mathbf{x}_i, t_i, y_i) +\alpha_t  \mathcal{L}_{treat}(t_i, \mathbf{z}_{t,i}, \mathbf{z}_{c,i})+\alpha_y  \mathcal{L}_{pred}(t_i, y_i, \mathbf{z}_{c,i}, \mathbf{z}_{y,i})\right] \\
    &+\alpha_1 \mathcal{L}_{indep}(\mathbf{z}_{t,*}, \mathbf{z}_{c,*}, \mathbf{z}_{y,*}, \mathbf{z}_{o,*})+\alpha_2   \mathcal{L}_{disc}(\mathbf{z}_{y,*})+ \lambda{\Vert \Theta \Vert}_2^2,
    \end{split}
\end{equation}
where $\alpha_t$, $\alpha_y$, $\alpha_1$, and $\alpha_2$ are non-negative hyperparameters that balance the corresponding terms. The final term, $\lambda{\Vert \Theta \Vert}_2^2$, is applied to all model parameters $\Theta$ to avoid overfitting. Despite the complexity of our loss function, the model runs efficiently and achieves strong performance by jointly optimizing all terms in a unified framework, as demonstrated in the experiments in the next section.

After completing the model training, we can predict the ITEs of new instances based on the observed covariates $\mathbf{x}$. We utilize the encoders $q_{\phi_c}(\mathbf{z}_c \mid \mathbf{x},\mathbf{A})$ and $q_{\phi_y}(\mathbf{z}_y \mid \mathbf{x},\mathbf{A})$ to sample the posteriors of confounding factors and risk factors $l$ times, and then use the decoder $p_{\theta_y}(y \mid t,\mathbf{z}_c, \mathbf{z}_y)$ to compute the predicted outcomes $y$ at different $t$, averaging them to obtain the estimated potential outcomes $y^1$ and $y^0$. The calculation of ATE can be done by performing the above steps on all test samples and then taking the average.
\section{Experiments}\label{sec:Experiments}
In this section, we perform a series of experiments to illustrate the effectiveness of the proposed TNDVGA framework. We first introduce the datasets, evaluation metrics, baselines, and model parameter configurations utilized in the experiments. Then, we compare the performance of different models in estimating ITE. After that, we conduct an ablation study to evaluate the importance of key components in the TNDVGA and conduct a hyperparameter study.
\subsection{Datasets}
The absence of ground-truth ITE is a widely recognized challenge, making the use of completely real datasets unfeasible. One common approach to address this is to generate synthetic data that simulates all potential outcomes for different treatments. To this end, we first construct a fully synthetic dataset. For more realistic validation, we then use two real-world networks with semi-synthetic data from \cite{guo2020learning} to further assess the model’s performance.
\subsubsection{Synthetic datasets}
Inspired by \cite{hassanpour2019learning}, we generate synthetic datasets named \textit{TNDVGASynth}, which follow the structure illustrated in Fig. \ref{fig:causal graph} and the relationships defined in Equations (\ref{eq:generate x})-(\ref{eq:generate y}).
\begin{equation}
    \begin{split}
        \mathbf{z}_t  \sim \mathcal{N}(\mathbf{0},\mathbf{1})^{m_t},\,\,\mathbf{z}_c \sim&\,\, \mathcal{N}(\mathbf{0},\mathbf{1})^{m_c},\,\,\mathbf{z}_y \sim \mathcal{N}(\mathbf{0},\mathbf{1})^{m_y},\,\,\mathbf{z}_o \sim \mathcal{N}(\mathbf{0},\mathbf{1})^{m_o}\,\,\\
        \mathbf{x}=&\,\,Concat(\mathbf{z}_t, \mathbf{z}_c, \mathbf{z}_y, \mathbf{z}_o),\\
        {\boldsymbol{\Psi}} =&\,\,Concat(\mathbf{z}_t, \mathbf{z}_c), \\
        {\boldsymbol{\Phi}} =&\,\,Concat(\mathbf{z}_c, \mathbf{z}_y),\\
    \end{split}
    \label{eq:generate x}
\end{equation}
\\
\begin{equation}
    \begin{split}
    &a\sim Bernoulli(\frac{0.01}{1+{\rm exp}(- r)})\\
    {\rm with}\,\,&r = \mathbf{h}\cdot\mathbf{h}+1,\,\, \mathbf{h}=Concatenate(\mathbf{z}_t, \mathbf{z}_c, \mathbf{z}_y, \mathbf{z}_o),
    \end{split}
    \label{eq:generate A}
\end{equation}
\\
\begin{equation}
    \begin{split}
        t&\sim Bernoulli(\frac{1}{1+{\rm exp}(-\zeta \mathbf{h})})\\
        {\rm with}\,&\, h={\boldsymbol{\Psi}}\cdot \boldsymbol{\theta}+1,\,\,\boldsymbol{\theta}\sim\mathcal{N}(\mathbf{0},\mathbf{1})^{m_t+m_c},
    \end{split}
    \label{eq:generate t}
\end{equation}
\\
\begin{equation}
    \begin{split}
        y^0&=\frac{{(\boldsymbol{\Phi}}\circ{\boldsymbol{\Phi}}\circ{\boldsymbol{\Phi}}+0.5)\cdot{\boldsymbol{\nu}}^0}{m_c+m_y}+\epsilon\\
        &y^1=\frac{{(\boldsymbol{\Phi}}\circ{\boldsymbol{\Phi}})\cdot{\boldsymbol{\nu}}^1}{m_c+m_y}+\epsilon\\
        {\rm with} \,\,&{\boldsymbol{\nu}}^0,{\boldsymbol{\nu}}^1\sim\mathcal{N}(\mathbf{0},\mathbf{1})^{m_c+m_y},\,\,\epsilon\sim\mathcal{N}(\mathbf{0},\mathbf{1}),
    \end{split}
    \label{eq:generate y}
\end{equation}
where $concat(\cdot, \cdot)$ represents the vector concatenation operation. $a$ is an element of the matrix $\mathbf{A}$; $m_t, m_c, m_y, m_o$ represent the dimensions of the latent factors $\mathbf{z}_t, \mathbf{z}_c, \mathbf{z}_y, \mathbf{z}_o$, respectively. A scalar 
$\zeta$ determines the slope of the logistic curve, and as 
$\zeta$ increases, the selection bias also increases. The symbol $\cdot$ denotes the dot product, and $\circ$ signifies the element-wise (Hadamard) product. We considered all feasible datasets generated from the grid defined by $m_t, m_c, m_y, m_o \in \{4, 8\}$, creating 16 scenarios. For each scenario, we synthesize five datasets using different initial random seeds.
\subsubsection{Semi-synthetic datasets}
\begin{table}[!b]
\belowrulesep=0pt
\aboverulesep=0pt
  \caption{Statistics of the Two Semi-Synthetic Datasets: BlogCatalog
and Flickr}
  \begin{tabular}{|c|c|c|c|c|c|}
    \toprule
    Datasets&Instances&Edges&Features&$\kappa_2$&ATE mean $\pm$ STD\\
    \midrule
    \multirow{3}{*}{BlogCatalog} & \multirow{3}{*}{5196} &\multirow{3}{*}{173,468}& \multirow{3}{*}{8,189}&0.5&4.366$\pm$0.553\\
    & & & &1&7.446$\pm$0.759\\
    & & & &2&13.534$\pm$2.309\\\hline
    \multirow{3}{*}{Flickr} & \multirow{3}{*}{7,575} & \multirow{3}{*}{239,738}&\multirow{3}{*}{12,047}&0.5&6.672$\pm$3.068\\
    & & & &1&8.487$\pm$3.372\\
    & & & &2&20.546$\pm$5.718\\
    \bottomrule
\end{tabular}
\label{tab:datasets}
\end{table}
\paragraph{BlogCatalog} In the BlogCatalog dataset \cite{tang2011leveraging}, a social blog directory for managing bloggers and their blogs, each individual represents a blogger, and each edge represents a social connection between two bloggers. The features are represented as a bag-of-words representation of the keywords in the bloggers' descriptions. To generate synthetic outcomes and treatments, we rely on the assumptions outlined in \cite{guo2020learning,veitch2019using}.
The outcome $y$ refers to the readers' opinion of each blogger, and the treatment $t$ represents whether the blogger's content receives more views on mobile devices or desktops. Bloggers whose content is primarily viewed on mobile devices are placed in the treatment group, while those whose content is mainly viewed on desktops are placed in the control group. Additionally, following the assumptions in \cite{guo2020learning}, we assume that the topics discussed by the blogger and their neighbors causally affect both the blogger’s treatment assignment and outcome. In this task, our goal is to investigate the ITE of receiving more views on mobile devices (instead of desktops) on the readers' opinion. 
We follow the synthetic procedure for \( t \) and \( y \) described in \cite{tang2011leveraging}. In this dataset, the hyperparameters \( \kappa_1 \geq 0 \) and \( \kappa_2 \geq 0 \) control the strength of the selection bias introduced by the blogger's topics and the topics of their neighbors, respectively. A higher value of \( \kappa_2 \) amplifies the influence of neighbors' topics on the device. For our study, we set \( \kappa_1 = 10 \) and \( \kappa_2 \in \{0.5, 1, 2\} \), thereby generating three distinct datasets.
\paragraph{Flickr} Flickr \cite{tang2011leveraging} is an online platform utilized for the purpose of sharing images and videos. In this dataset, each user is represented as an instance, with edges indicating social connections between users. The features of each user are a list of interest tags. 
The treatment and outcome are synthesized using the same settings and simulation process as in the BlogCatalog.

In Table \ref{tab:datasets}, we provide a detailed statistical summary of the two semi-synthetic datasets. For each parameter setting, the mean and standard deviation of the ATEs are computed across 10 runs.
\subsection{Evaluation Metrics}
We evaluate the performance of the proposed TNDVGA framework in learning ITE using two widely used metrics in causal inference. We report the Root Precision in Estimating Heterogeneous Effects ($\sqrt{\epsilon_{PEHE}}$) to measure the accuracy of individual-level treatment effect, and the Mean Absolute Error of ATE ($\epsilon_{ATE}$) to assess the accuracy of population-level treatment effect. They are formally defined as follows:
\begin{align}
    &\sqrt{\epsilon_{PEHE}}=\sqrt{\frac{1}{n}\sum_{i=1}^n({\hat{\tau}}_i-\tau_i)^2},\\
    &\epsilon_{ATE}=\frac{1}{n}\lvert\sum_{i=1}^n{\hat{\tau}}_i-\sum_{i=1}^n\tau_i\rvert,
\end{align}
where ${\hat{\tau}}_i={\hat{y}}_i^1-{\hat{y}}_i^0$ and $\tau_i=y_i^1-y_i^0$ represent the estimated ITE and the ground-truth ITE of instance $i$, respectively. Lower values of these metrics indicate better estimating performance.
\subsection{Baselines}
We compare our model against the following state-of-the-art models used for ITE estimation:
\begin{itemize}
\item \textbf{Bayesian Additive Regression Trees (BART)} \cite{chipman2010bart}. BART is a widely used nonparametric Bayesian regression model that utilizes dimensionally adaptive random basis functions.
\item \textbf{Causal Forest} \cite{wager2018estimation}. Causal Forest is a non-parametric causal inference method designed to estimate heterogeneous treatment effects, extending Breiman's well-known random forest algorithm.
\item \textbf{Counterfactual Regression (CFR)} \cite{shalit2017estimating}. CFR is a representation learning-based approach for predicting individual treatment effects from observational data. It reduces imbalance between treatment and control groups' latent representations and minimizes prediction errors by projecting original features into a latent space to capture confounders. CFR uses Integral Probability Metrics to measure distribution distances. This study employs two balancing penalties: Wasserstein-1 distance (CFR-Wass) and maximum mean discrepancy (CFR-MMD).
\item \textbf{Treatment-agnostic Representation Networks (TARNet)} \cite{shalit2017estimating}. TARNet is a variant of CFR that excludes the balance regularization term from its model.
\item \textbf{Causal Effect Variational Autoencoder
(CEVAE)} \cite{louizos2017causal}. CEVAE is built upon Variational Autoencoders (VAE) \cite{kingma2013auto} and adheres to the causal inference framework with proxy variables. It is capable of jointly estimating the unknown latent space that captures confounders and the causal effect.
\item \textbf{Identifiable treatment-conditional VAE (Intact-VAE)} \cite{wu2021intact}. Intact-VAE is a novel VAE architecture designed to estimate causal effects in the presence of unobserved confounding.
\item \textbf{Treatment Effect by Disentangled Variational AutoEncoder (TEDVAE)} \cite{zhang2021treatment}. TEDVAE is a variational inference approach that simultaneously infers latent factors from observed variables, while disentangling these factors into three distinct sets: instrumental factors, confounding factors, and risk factors. These disentangled factors are then utilized for estimating treatment effects.
\item \textbf{Network Deconfounder (NetDeconf)} \cite{guo2020learning}. NetDeconf is a novel causal inference framework that leverages network information to identify patterns of latent confounders, enabling the learning of valid individual causal effects from networked observational data.
\item \textbf{Graph Infomax Adversarial Learning (GIAL)} \cite{chu2021graph}.
GIAL is a model for estimating treatment effects that uses network structure to capture additional information by identifying imbalances within the network. In this work, we use two variants of GIAL: one with the original GCN implementation \cite{kipf2016semi} (GIAL-GCN) and another with GAT \cite{velivckovic2017graph} (GIAL-GAT).
\end{itemize}
\subsection{Parameter Settings}
We implement TNDVGA using PyTorch on an NVIDIA RTX 4090D GPU. For BlogCatalog and Flickr, we run 10 experiments and report the average results. The dataset is split into training (60\%), validation (20\%), and test (20\%) sets for each run. Baseline methods, such as BART, Causal Forest, CFR, TARNet, and CEVAE, are originally designed for non-networked observational data and thus cannot leverage network information directly. We concatenate the adjacency matrix rows with the original features to ensure a fair comparison; however, this does not notably enhance baseline performance due to dimensionality limitations. For the baselines, we used the default hyperparameters as in previous works \cite{guo2020learning, chu2021graph}. For TNDVGA, we apply grid search to identify the optimal hyperparameter settings. Specifically, the learning rate is set to $3 \times 10^{-4}$, $\alpha_t$ and $\alpha_y$ set to 100, and $\lambda$ to $5 \times 10^{-5}$. The number of GCN layers is varied between 1, 2, and 3, with the hidden dimensions set to 500, and the dimensions of $\mathbf{z}_t$, $\mathbf{z}_c$, $\mathbf{z}_y$, and $\mathbf{z}_o$ vary across \{10, 20, 30, 40, 50\}. The regularization coefficients $\alpha_1$ and $\alpha_2$ are tuned within the range \{$10^{-2}$, $10^{-1}$, 1, $10$, $100$\}. For the BlogCatalog dataset, TNDVGA is trained for 500 epochs, while for Flickr, training lasts for 1000 epochs. And we use Adam optimizer \cite{kingma2014adam} to train TNDVGA. For the synthetic datasets, we use the same parameter selection approach as for the semi-synthetic datasets. Unless stated otherwise, the latent variable dimensions for the different factors are set to their true values. The size of each synthetic dataset is identical, with 5000 instances for training, 1000 for validation, and 1000 for testing. 
\subsection{Perfomance Comparision}
We compare the proposed framework TNDVGA with the state-of-the-art baselines for ITE estimation on both semi-synthetic datasets and synthetic datasets. For the semi-synthetic dataset, since the experimental settings are the same as those in \cite{chu2021graph}, we use the baseline results from that paper, except for Intact-VAE and TEDVAE.
\subsubsection{Performance on Synthetic datasets}
\begin{figure}[b]
  \centering
  \includegraphics[width=0.4\linewidth]{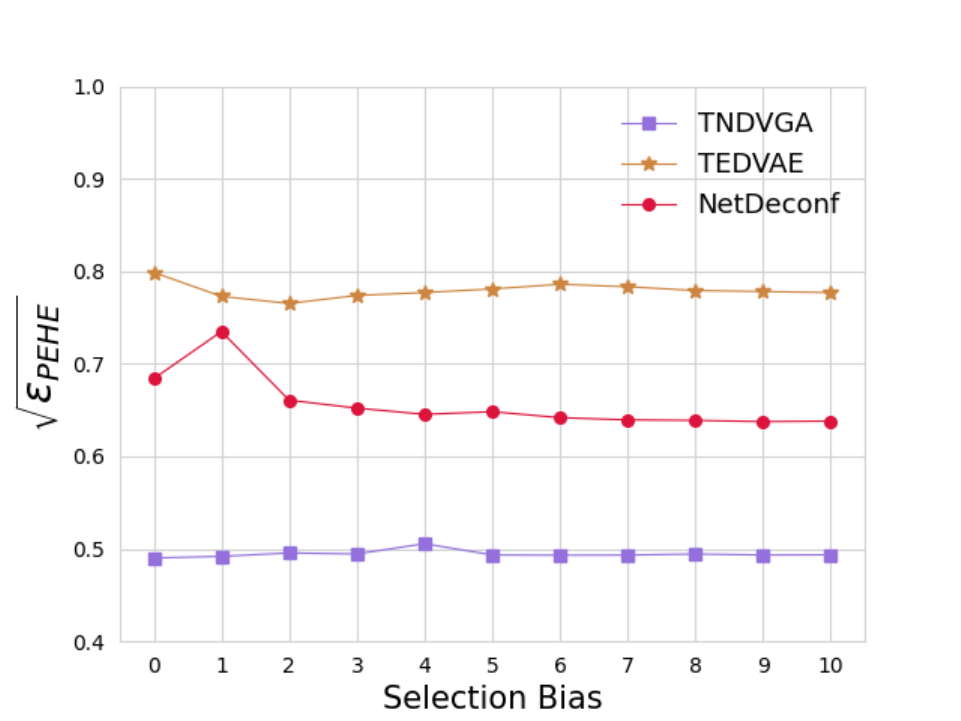}
  \caption{Experimental results of different methods in ITE estimation under different levels of selection bias. As the selection bias increases, TNDVGA consistently performs the best.
}
  \Description{syn_compare}
  \label{fig:selection bias}
\end{figure}
\begin{figure}[t]
	\centering
	\begin{minipage}{0.22\linewidth}
		\centering
		\includegraphics[width=1.0\linewidth]{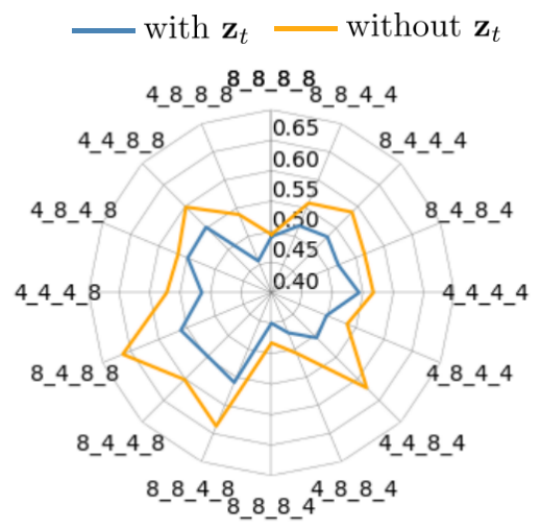}
		\label{fig:syn_zt}
	\end{minipage}
	\begin{minipage}{0.22\linewidth}
		\centering
		\includegraphics[width=1.0\linewidth]{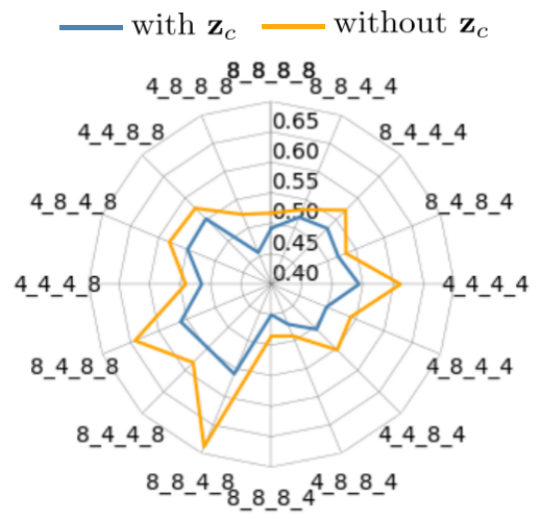}
		\label{fig:syn_zc}
	\end{minipage}
    \begin{minipage}{0.22\linewidth}
		\centering
		\includegraphics[width=1.0\linewidth]{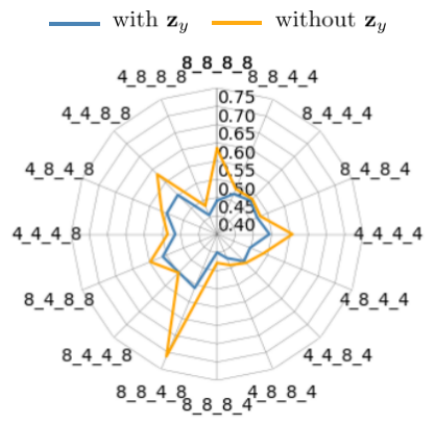}
		\label{fig:syn_zy}
	\end{minipage}
    \begin{minipage}{0.22\linewidth}
		\centering
		\includegraphics[width=1.0\linewidth]{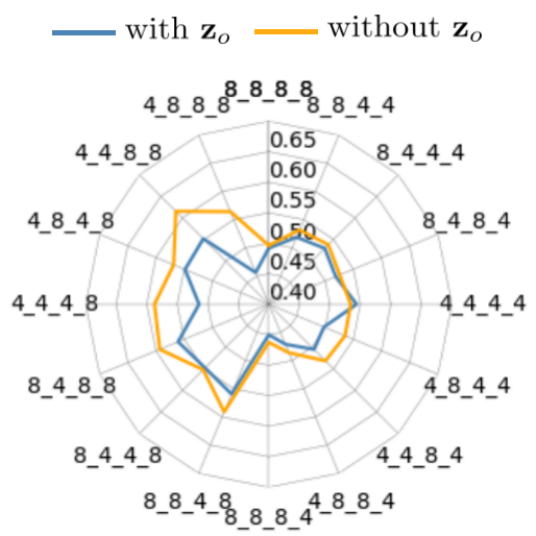}
		\label{fig:syn_zo}
	\end{minipage}
 \caption{In the radar chart, each vertex of the polygon is labeled with a sequence of latent factor dimensions from the synthetic dataset. For example, \( 8\text{-}8\text{-}8\text{-}8 \) indicates that the dataset is generated using 8 dimensions each for latent instrumental factors, latent confounding factors, latent adjustment factors, and latent noise factors. Each polygon represents the PEHE metric of the model (smaller polygons indicate better performance).
}
 \label{fig:radar charts}
\end{figure}
First, similar to \cite{bao2022learning}, we control the magnitude of selection bias in the dataset by setting the size of the scalar $\zeta$. We compare TNDVGA with TEDVAE and NetConf when the dimensions of the latent factors are (8, 8, 8, 8). As shown in Fig. \ref{fig:selection bias}, we observe that as the value of $\zeta$ increases, indicating a rise in selection bias, TNDVGA consistently performs the best. Furthermore, TNDVGA's performance remains stable and is largely unaffected by variations in selection bias. This demonstrates that TNDVGA exhibits better robustness against selection bias, which is crucial when handling real-world datasets. We have similar results on other synthetic datasets.

Next, we investigate the TNDVGA's ability to recover the latent components $\mathbf{z}_t$, $\mathbf{z}_c$, $\mathbf{z}_y$, and $\mathbf{z}_o$ that are utilized to construct the observed covariates $\mathbf{x}$ and examine the contribution of disentangling these latent factors to the estimation of ITE. To this end, similar to the settings in \cite{hassanpour2019learning,bao2022learning}, we compare the performance of TNDVGA when the parameters $d_{\mathbf{z}_t}$, $d_{\mathbf{z}_c}$, $d_{\mathbf{z}_y}$ and $d_{\mathbf{z}_o}$ are set to correspond with the true number of latent factors against the performance when one of the latent dimensionality parameters is set to zero. For example, setting $d_{\mathbf{z}_c} = 0$ forces TNDVGA to ignore the disentanglement of confounding factors. If TNDVGA performs better when considering the disentanglement of all latent factors compared to when any one latent factor is ignored, then it can be concluded that TNDVGA can recover latent factors, and that disentangling latent factors is beneficial for ITE estimation. Fig. \ref{fig:radar charts} displays the radar charts corresponding to each factor. We can clearly see that when TNDVGA considers disentangling the latent factors through non-zero dimensionality parameters, its performance outperforms that when any latent dimension is set to zero. 
\subsubsection{Performance on Semi-Synthetic datasets} 
Tables \ref{tab:BlogCatalog} and \ref{tab:Flickr} present the experimental results on the BlogCatalog and Flick datasets, respectively. Through a comprehensive analysis of the experimental results, we have the following observations:
\begin{itemize}
    \item The proposed variational inference framework for ITE estimation, TNDVGA, consistently outperforms state-of-the-art traditional baseline methods, including BART, Causal Forest, CFR, CEVAE and Intact-VAE, across different settings on both datasets, as these methods do not account for disentangled latent factors or leverage network information for ITE learning. 
    \item TNDVGA and NetDeconf, along with GIAL, outperform other baseline methods in ITE estimation due to their ability to leverage auxiliary network information to capture the impact of latent factors on ITE estimation. This result suggests that network information helps in learning representations of latent factors, leading to more accurate ITE estimation. Furthermore, TNDVGA also outperforms NetDeconf and GIAL in ITE estimation because it learns representations of four different latent factors, whereas NetDeconf and GIAL only learn representations of latent confounding factors.
    \item TEDVAE also performs reasonably well in estimating ITE, mainly because its model infers and disentangles three disjoint sets of instrumental, confounding, and risk factors from the observed variables. This also highlights the importance of learning disentangled latent factors for ITE estimation. However, TNDVGA outperforms TEDVAE, as it additionally accounts for latent noise factors and effectively leverages network information, whereas TEDVAE struggles to fully utilize network information to enhance its modeling capabilities.
    \item TNDVGA demonstrates strong robustness in selecting the latent dimensionality parameter. Considering that we did not explicitly model the generation process of latent factors in these two semi-synthetic real datasets, 
    TNDVGA still exhibits optimal performance under these conditions. These results indicate that, even in more realistic datasets 
    , TNDVGA can effectively learn latent factors and estimate ITE.
    \item When the influence of latent confounders increases (i.e., with a growing $\kappa_2$ value), TNDVGA suffers the least in $\sqrt{\epsilon_{PEHE}}$ and $\epsilon_{ATE}$. This is because TNDVGA has the ability to identify patterns of latent confounding factors from the network structure, enabling it to infer ITE more accurately.
    
\end{itemize}
\begin{table}[!t]
\belowrulesep=0pt
\aboverulesep=0pt
  \caption{Performance comparison for different methods on BlogCatalog. We report the average values of \( \sqrt{\epsilon_{PEHE}} \) and $\epsilon_{ATE}$ on the test sets. 
}
  \label{tab:BlogCatalog}
  \begin{tabular}{|l|c|c|c|c|c|c|}
    \toprule
    \multicolumn{7}{|c|}{BlogCatalog}\\\hline
    $\kappa_2$&\multicolumn{2}{c|}{0.5}&\multicolumn{2}{c|}{1}&\multicolumn{2}{c|}{2}\\\hline
    &$\sqrt{\epsilon_{PEHE}}$&$\epsilon_{ATE}$&$\sqrt{\epsilon_{PEHE}}$&$\epsilon_{ATE}$&$\sqrt{\epsilon_{PEHE}}$&$\epsilon_{ATE}$\\
    \midrule
    BART  & 4.808 & 2.680 & 5.770 & 2.278 & 11.608 & 6.418  \\ \hline
        Causal Forest & 7.456 & 1.261 & 7.805 & 1.763 & 19.271 & 4.050   \\ \hline
        CFR-Wass & 10.904 & 4.257 & 11.644 & 5.107 & 34.848 & 13.053 \\ \hline
        CFR-MMD & 11.536 & 4.127 & 12.332 & 5.345 & 34.654 & 13.785\\ \hline
        TARNet  & 11.570 & 4.228 & 13.561 & 8.170 & 34.420 & 13.122 \\ \hline
        CEVAE & 7.481 & 1.279 & 10.387 & 1.998 & 24.215 & 5.566  \\ \hline
        Intact-VAE & 6.260 & 1.925 & 6.840 & 0.845 & 13.618 & 3.329 \\ \hline
        TEDVAE & 4.609 & 0.798 & 4.354 & 0.881 & 6.805 & 1.190 \\ \hline
        NetDeconf & 4.532 & 0.979 & 4.597 & 0.984 & 9.532 & 2.130\\ \hline
        GIAL-GCN & 4.023 & 0.841 & 4.091 & 0.883 & 8.927 & 1.780  \\\hline
        GIAL-GAT &4.215&0.912& 4.258 &0.937& 9.119& 1.982\\ \hline
        TNDVGA (ours) &\textbf{3.969}&\textbf{0.719}&\textbf{3.846}&\textbf{0.699}&\textbf{6.066}&\textbf{1.057}\\
    \bottomrule
\end{tabular}
\end{table}
\begin{table}[!t]
\belowrulesep=0pt
\aboverulesep=0pt
  \caption{Performance comparison for different methods on Flickr. We report the average values of \( \sqrt{\epsilon_{PEHE}} \) and $\epsilon_{ATE}$ on the test sets. }
  \label{tab:Flickr}
  \begin{tabular}{|l|c|c|c|c|c|c|}
    \toprule
    \multicolumn{7}{|c|}{Flickr}\\\hline
    $\kappa_2$&\multicolumn{2}{c|}{0.5}&\multicolumn{2}{c|}{1}&\multicolumn{2}{c|}{2}\\\hline
    &$\sqrt{\epsilon_{PEHE}}$&$\epsilon_{ATE}$&$\sqrt{\epsilon_{PEHE}}$&$\epsilon_{ATE}$&$\sqrt{\epsilon_{PEHE}}$&$\epsilon_{ATE}$\\
    \midrule
    BART  & 4.907 & 2.323 & 9.517 & 6.548 & 13.155 & 9.643  \\ \hline
        Causal Forest  & 8.104 & 1.359 & 14.636 & 3.545 & 26.702 & 4.324  \\ \hline
        CFR-Wass  & 13.846 & 3.507 & 27.514 & 5.192 & 53.454 & 13.269  \\ \hline
        CFR-MMD  & 13.539 & 3.350 & 27.679 & 5.416 & 53.863 & 12.115 \\ \hline
        TARNet  & 14.329 & 3.389 & 28.466 & 5.978 & 55.066 & 13.105 \\ \hline
        CEVAE  & 12.099 & 1.732 & 22.496 & 4.415 & 42.985 & 5.393 \\ \hline
        Intact-VAE & 6.483 & 0.635 & 10.721 & 1.085 & 16.050 & 2.022 \\\hline
        TEDVAE & 5.072 & 1.041 & 7.125 & 1.328 & 12.952 & 2.124 \\\hline
        NetDeconf  & 4.286 & 0.805 & 5.789 & 1.359 & 9.817 & 2.700 \\ \hline
        GIAL-GCN & 3.938 & 0.682 & 5.317 & 1.194 & 9.275 & 2.245 \\\hline
        GIAL-GAT& 4.015& 0.773& 5.432& 1.231& 9.428 &2.586\\ \hline
        TNDVGA (ours) &\textbf{3.896}&\textbf{0.633}&\textbf{4.974}&\textbf{1.037}&\textbf{7.302}&\textbf{1.908}\\
    \bottomrule
\end{tabular}
\end{table}
\begin{table}[!b]
\belowrulesep=0pt
\aboverulesep=0pt
  \caption{Ablation study of our method's variants on BlogCatalog.}
  \label{tab:BlogCatalog_ablation}
  \begin{tabular}{|c|c|c|c|c|c|c|}
    \toprule
    \multicolumn{7}{|c|}{BlogCatalog}\\\hline
    $\kappa_2$&\multicolumn{2}{c|}{0.5}&\multicolumn{2}{c|}{1}&\multicolumn{2}{c|}{2}\\\hline
    &$\sqrt{\epsilon_{PEHE}}$&$\epsilon_{ATE}$&$\sqrt{\epsilon_{PEHE}}$&$\epsilon_{ATE}$&$\sqrt{\epsilon_{PEHE}}$&$\epsilon_{ATE}$\\
    \midrule
        TNDVGA &3.937&0.656&3.918&0.677&0.651&1.184\\\hline
        TNDVGA(w/o BP)& 4.090&0.710&4.060&0.808&6.887&1.798\\\hline
        TNDVGA(w/o HSIC)&4.114&0.765&4.070&0.808&6.982&1.958\\
    \bottomrule
\end{tabular}
\label{tab:ablation study1}
\end{table}
\begin{table}[htbp]
\belowrulesep=0pt
\aboverulesep=0pt
  \caption{Ablation study of our method's variants on Flickr.}
  \label{tab:Flickr_ablation}
  \begin{tabular}{|c|c|c|c|c|c|c|}
    \toprule
        \multicolumn{7}{|c|}{Flickr}\\\hline
    $\kappa_2$&\multicolumn{2}{c|}{0.5}&\multicolumn{2}{c|}{1}&\multicolumn{2}{c|}{2}\\\hline
    &$\sqrt{\epsilon_{PEHE}}$&$\epsilon_{ATE}$&$\sqrt{\epsilon_{PEHE}}$&$\epsilon_{ATE}$&$\sqrt{\epsilon_{PEHE}}$&$\epsilon_{ATE}$\\
    \midrule
    TNDVGA &3.897&0.610&5.045&0.956&8.763&1.074\\\hline
        TNDVGA(w/o BP) &4.298&0.637&5.551&1.359&10.8531&1.678\\\hline
        TNDVGA(w/o HSIC)&4.622&0.930&5.908&1.380&11.198&1.948\\
    \bottomrule
\end{tabular}
\label{tab:ablation study2}
\end{table}
\subsection{Ablation Study}
Furthermore, we investigate the effect of key components in the proposed TNDVGA framework on learning ITE from network observational data. In particular, we conduct an ablation study by developing two variants of TNDVGA and comparing their performance on the BlogCatalog and Flickr datasets with the original TNDVGA: (i) TNDVGA w/o Balanced Representations: This variant does not balance the learned representations, meaning it does not include the balanced representation loss $\mathcal{L}_{disc}$ during training. As a result, the learning factor $\mathbf{z}_y$ may embed information about $\mathbf{z}_t$. We refer to this variant as TNDVGA w/o BP.
 (ii) TNDVGA w/o HSIC Independence Regularizer: This variant omits the independence constraint mechanism between different factor representations, which may prevent the learned representations from being disentangled. We refer to this variant as TNDVGA w/o HSIC.
\begin{figure}[!t]
\centering
\subfigure[$\sqrt{\epsilon_{PEHE}}(\kappa_2=0.5)$]{
\includegraphics[width=3.5cm,trim=70 30 85 35,clip]{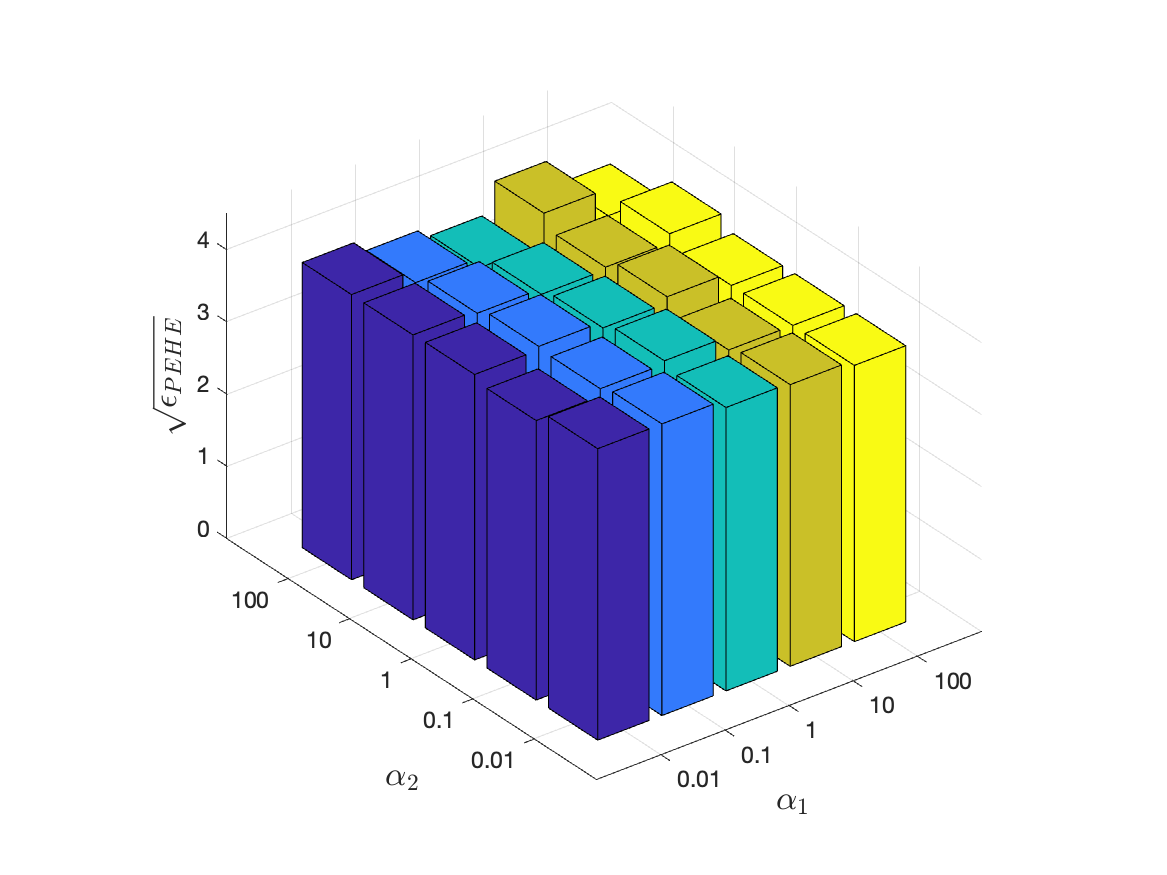}
}
\quad
\subfigure[$\sqrt{\epsilon_{PEHE}}(\kappa_2=1)$]{
\includegraphics[width=3.5cm,trim=70 30 85 35,clip]{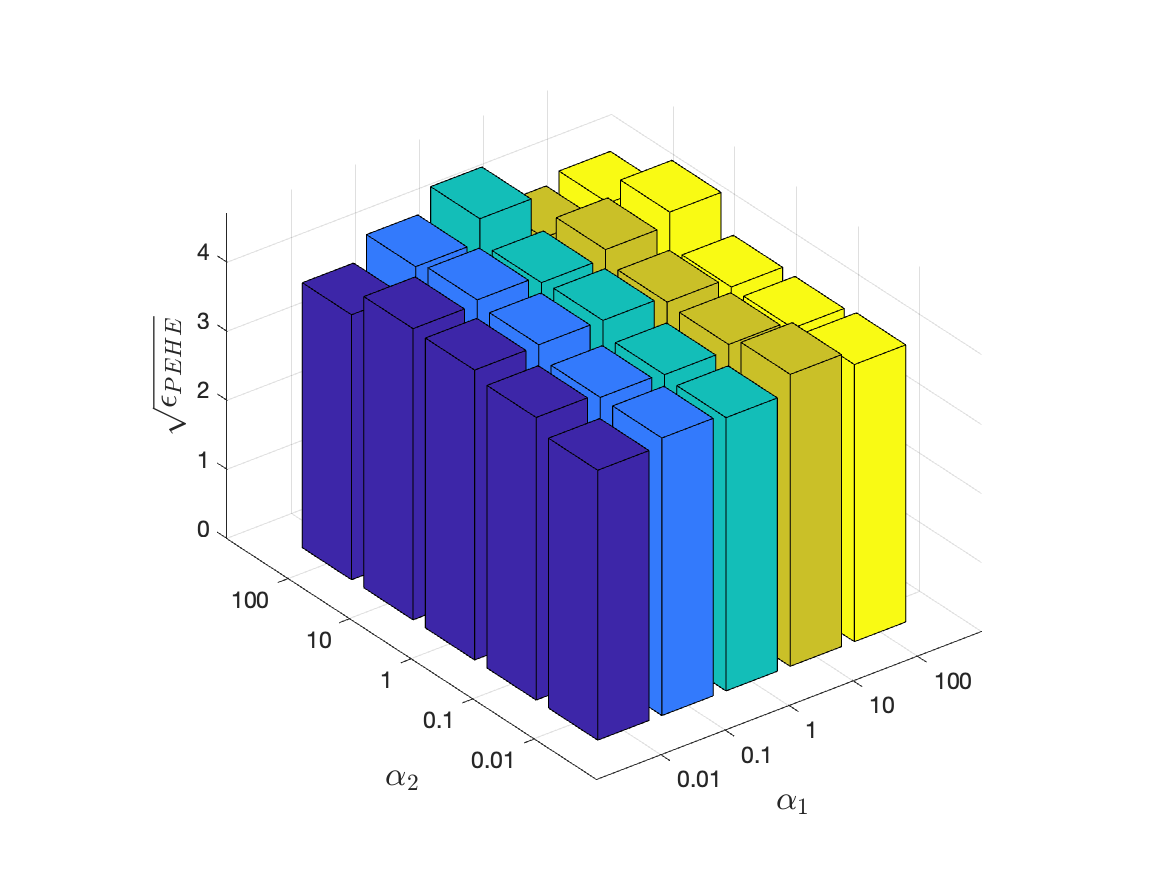}
}
\quad
\subfigure[$\sqrt{\epsilon_{PEHE}}(\kappa_2=2)$]{
\includegraphics[width=3.5cm,trim=70 30 85 35,clip]{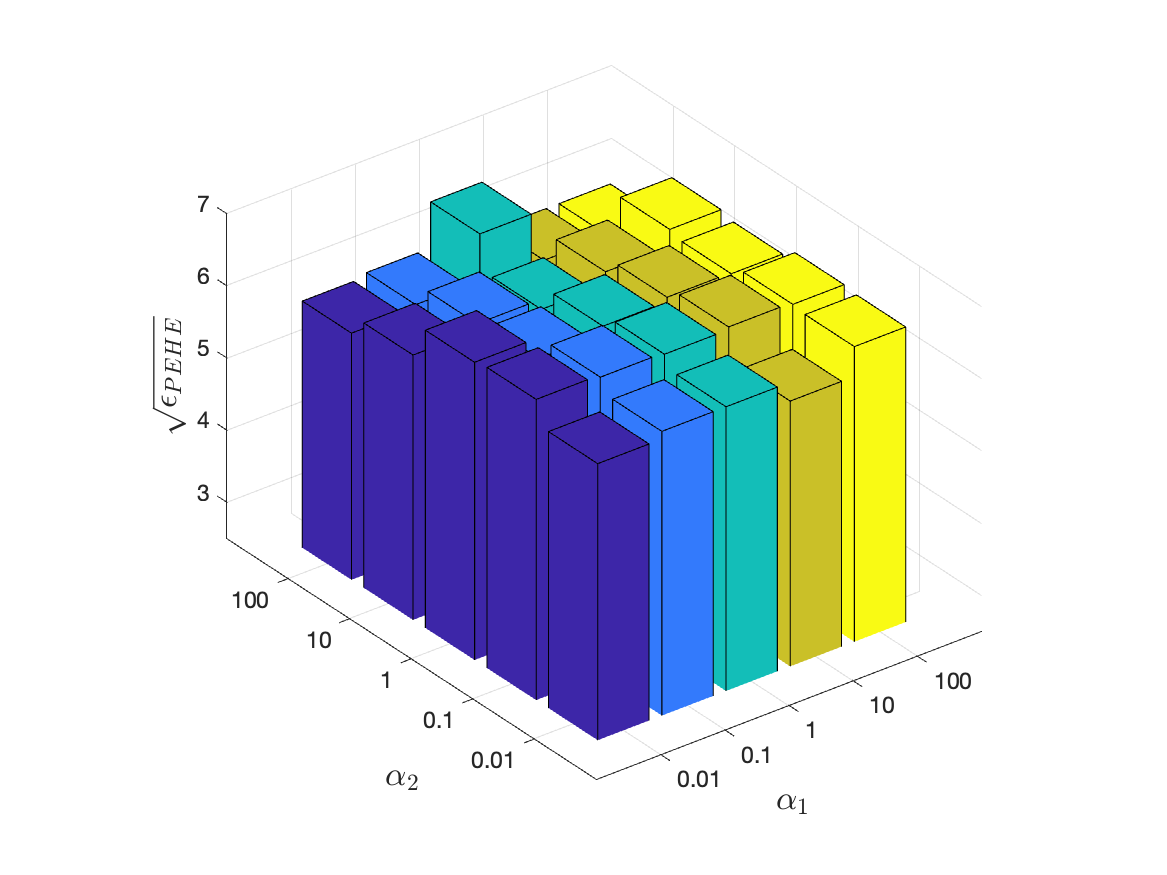}
}
\quad
\subfigure[${\epsilon_{ATE}}(\kappa_2=0.5)$]{
\includegraphics[width=3.5cm,trim=70 30 85 35,clip]{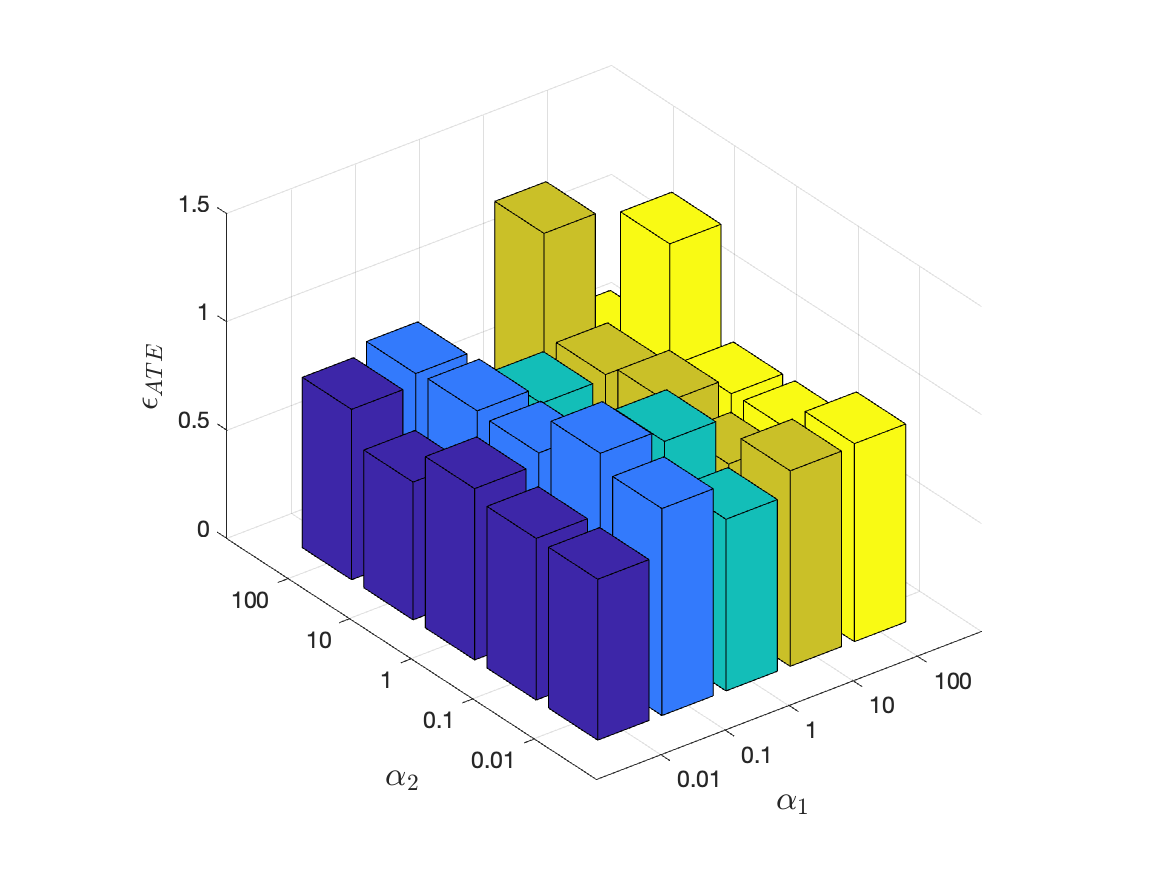}}
\quad
\subfigure[${\epsilon_{ATE}}(\kappa_2=1)$]{
\includegraphics[width=3.5cm,trim=70 30 85 35,clip]{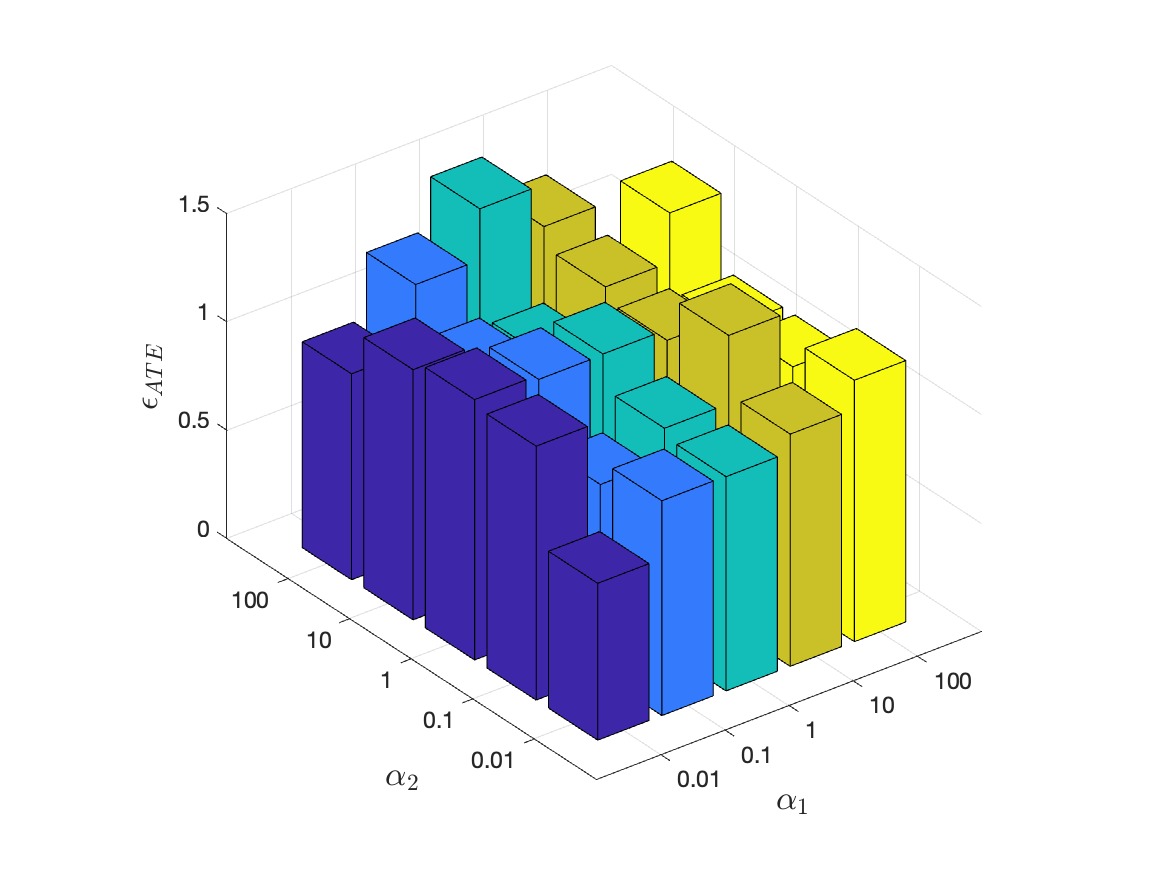}}
\quad
\subfigure[${\epsilon_{ATE}}(\kappa_2=2)$]{
\includegraphics[width=3.5cm,trim=70 30 85 35,clip]{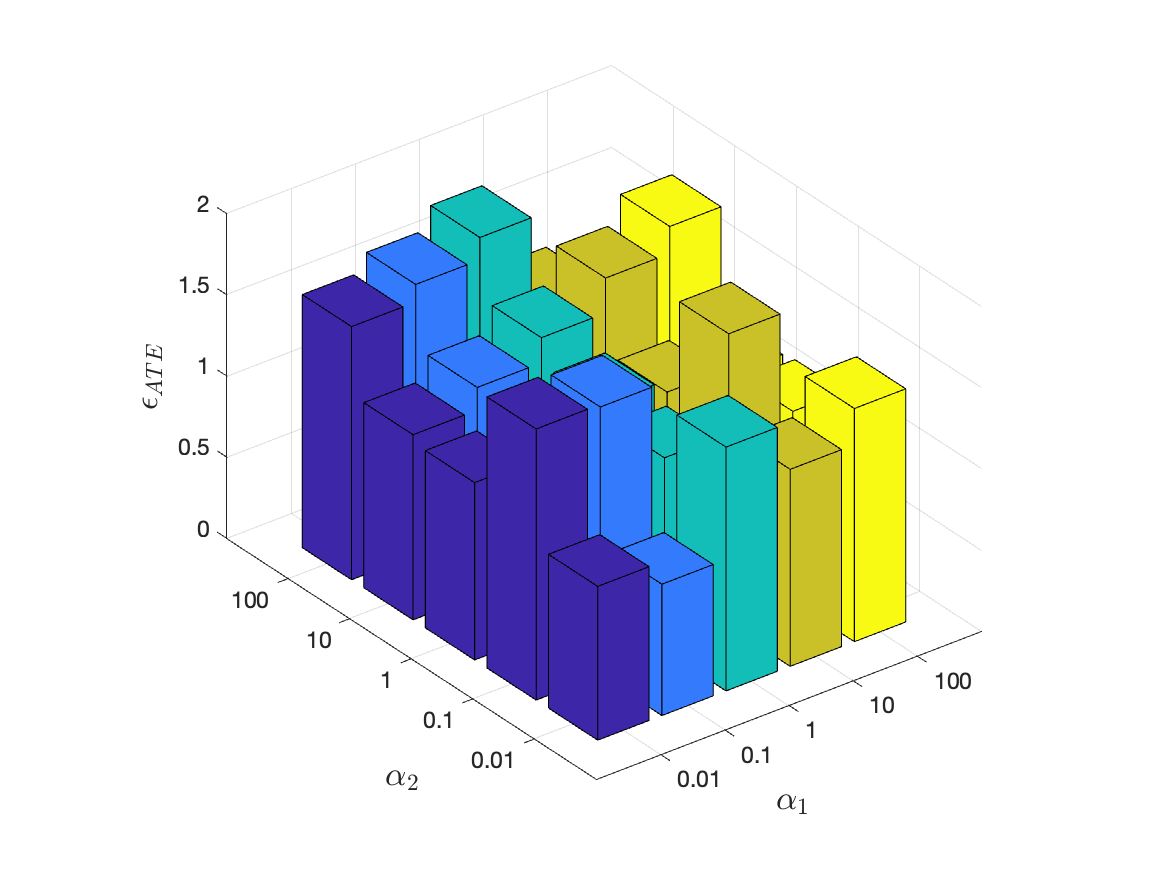}}
\caption{Hyperparameter analysis on BlogCatalog across different $\kappa_2$.}
\Description{Hyperparameter analysis}
\label{fig:hyperparameter study-Blogcatalog}
\end{figure}
\begin{figure}[t]
\centering
\subfigure[$\sqrt{\epsilon_{PEHE}}(\kappa_2=0.5)$]{
\includegraphics[width=3.5cm,trim=70 30 85 35,clip]{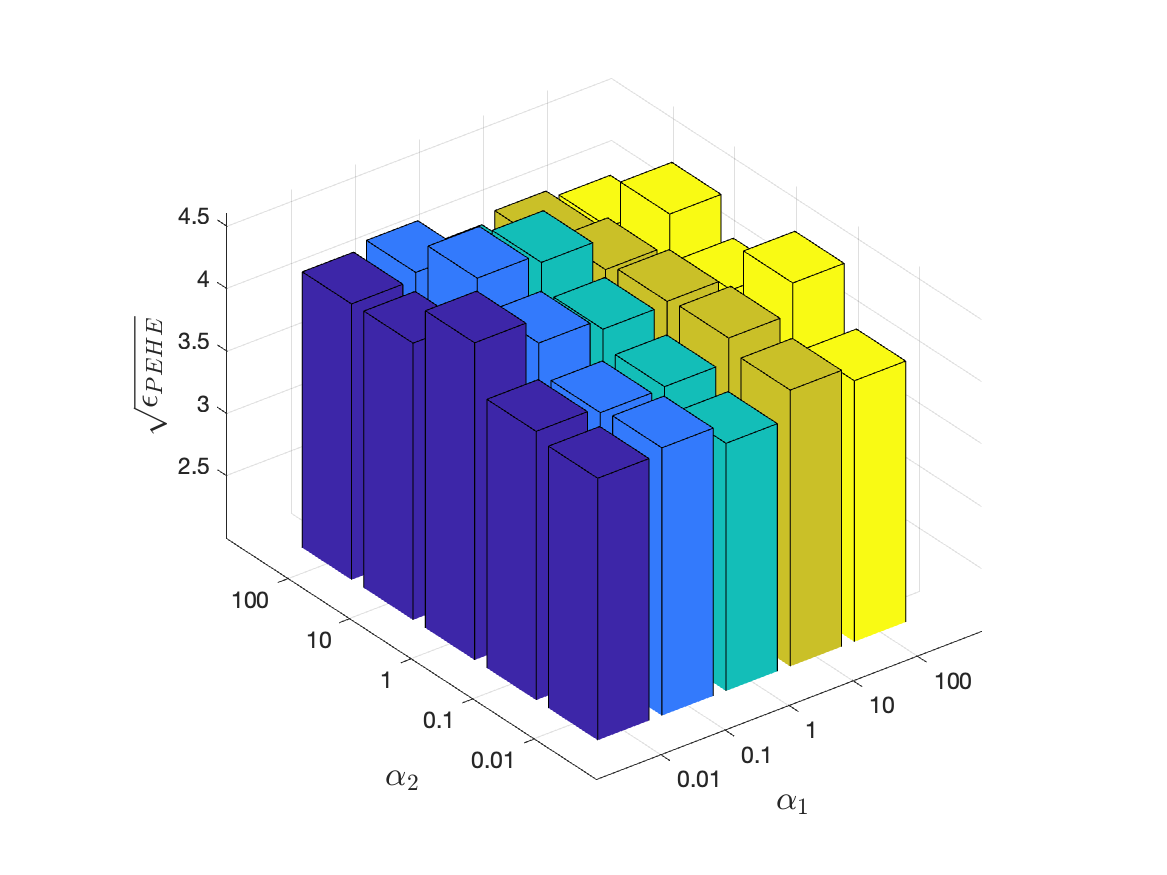}
}
\quad
\subfigure[$\sqrt{\epsilon_{PEHE}}(\kappa_2=1)$]{
\includegraphics[width=3.5cm,trim=70 30 85 35,clip]{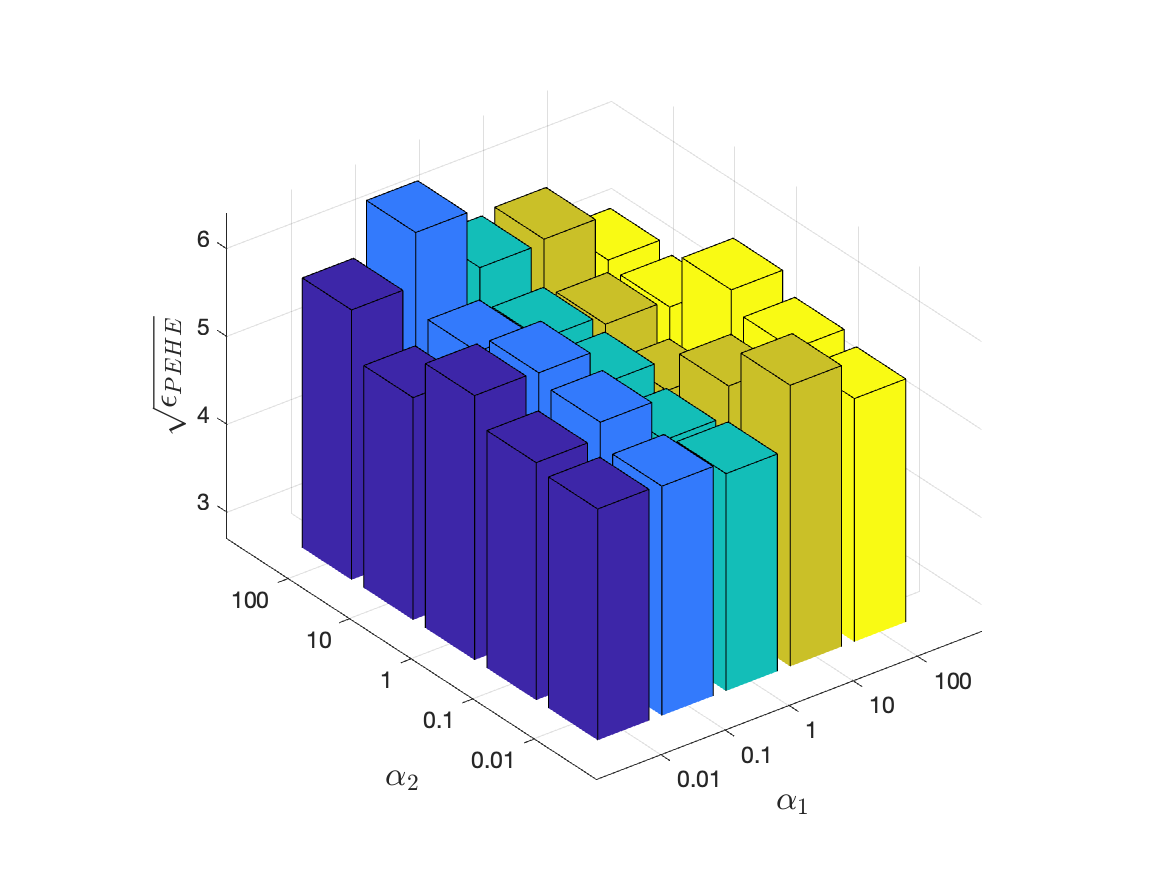}
}
\quad
\subfigure[$\sqrt{\epsilon_{PEHE}}(\kappa_2=2)$]{
\includegraphics[width=3.5cm,trim=70 30 85 35,clip]{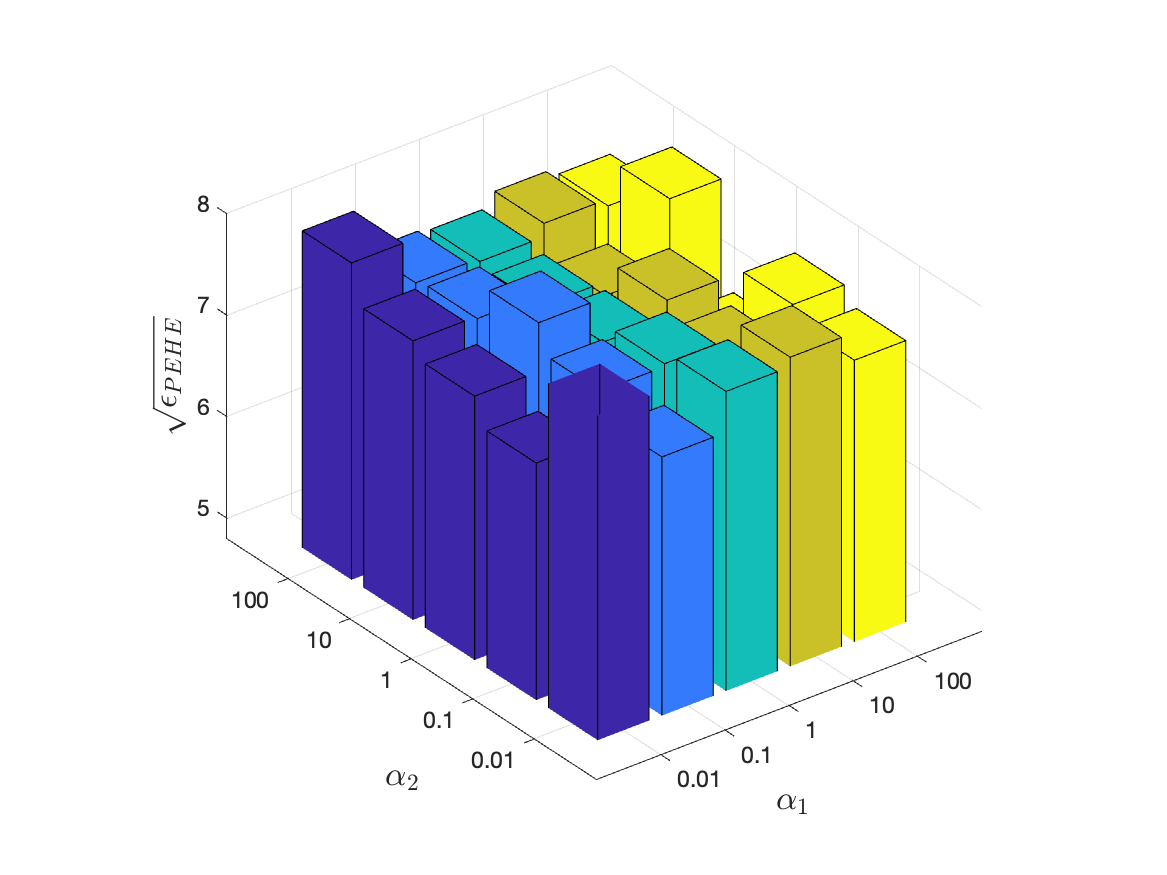}
}
\quad
\subfigure[${\epsilon_{ATE}}(\kappa_2=0.5)$]{
\includegraphics[width=3.5cm,trim=70 30 85 35,clip]{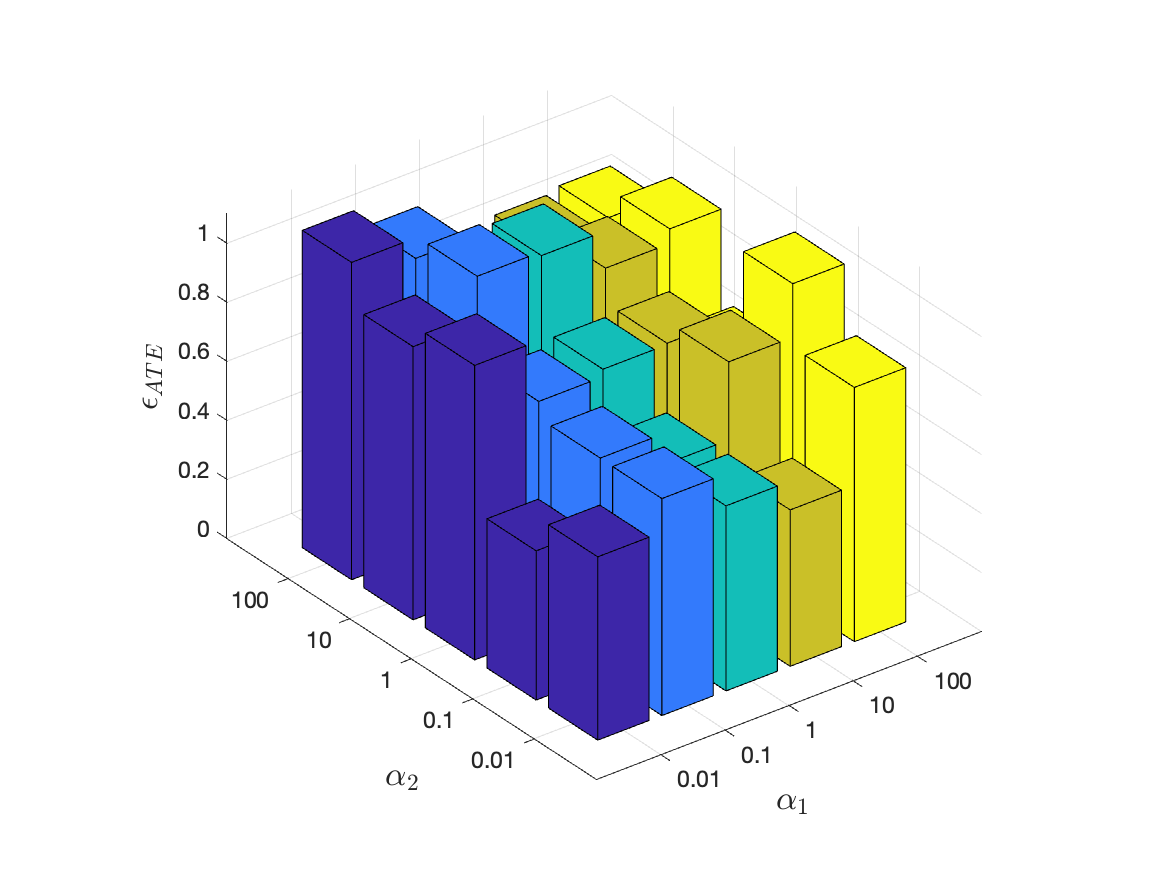}}
\quad
\subfigure[${\epsilon_{ATE}}(\kappa_2=1)$]{
\includegraphics[width=3.5cm,trim=70 30 85 35,clip]{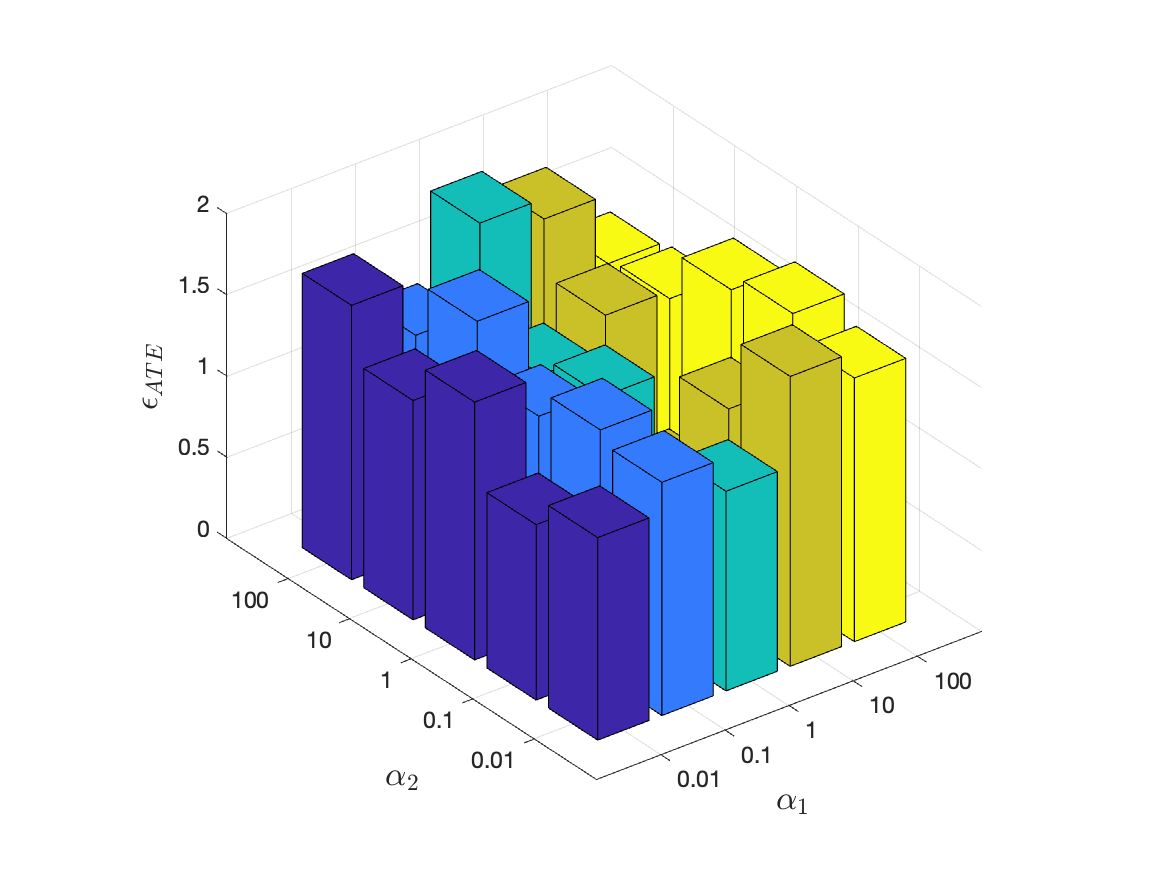}}
\quad
\subfigure[${\epsilon_{ATE}}(\kappa_2=2)$]{
\includegraphics[width=3.5cm,trim=70 30 85 35,clip]{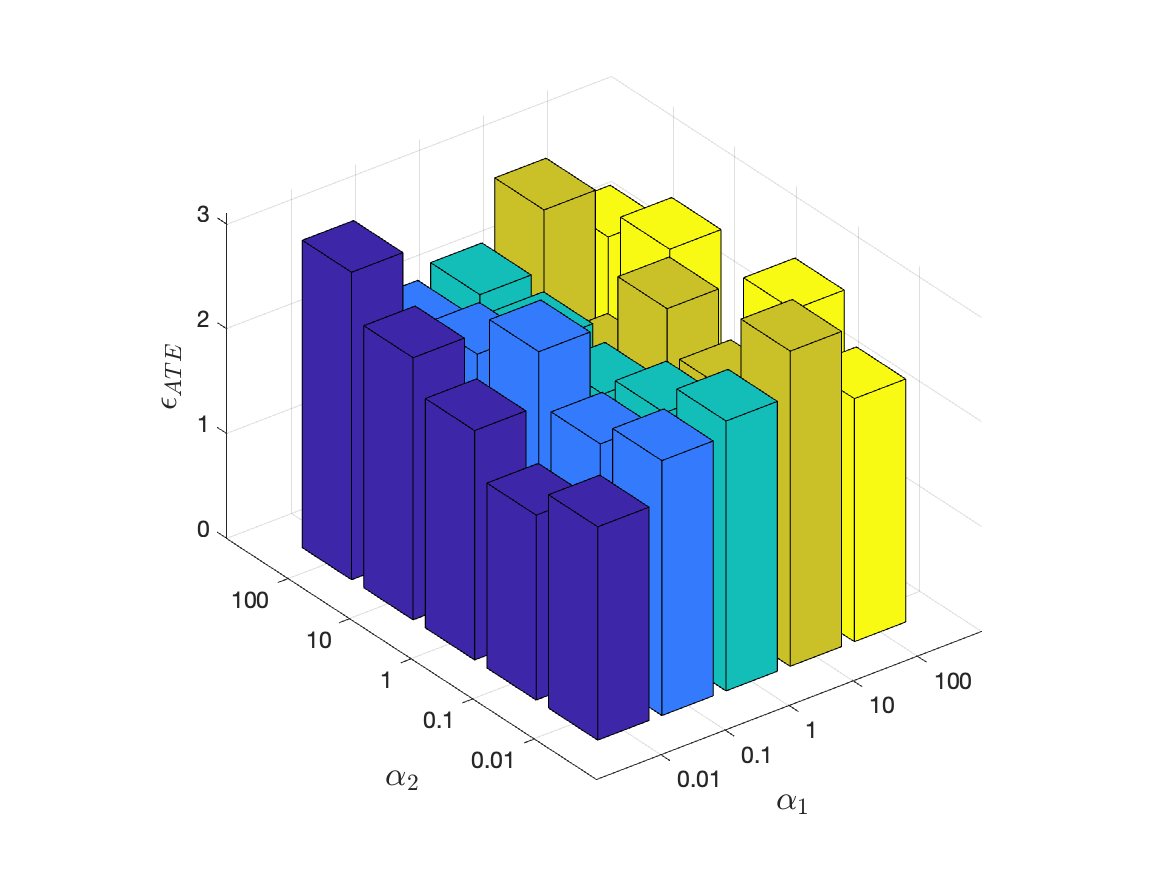}}
\caption{Hyperparameter analysis on Flickr across different $\kappa_2$.}
\Description{Hyperparameter analysis Flickr}
\label{fig:hyperparameter study-Flickr}
\end{figure}
Tables \ref{tab:ablation study1} and \ref{tab:ablation study2} display the comparison results of the two variants with TNDVGA on the BlogCatalog and Flickr datasets, respectively. From the analysis, we can draw the following observations:
\begin{itemize}
    \item TNDVGA w/o BP cannot provide satisfactory performance because it neglects the balance of adjustment variables, which may lead to instrumental information being embedded in the adjustment variables, affecting the effectiveness of the learned representations. This highlights the necessity of balanced representations for better learning of latent factors in order to estimate ITE. 
    \item TNDVGA w/o HSIC also fails to provide the expected performance  and typically performs the worst, as it does not impose independence constraints on the representations corresponding to different latent factors. This indicates that imposing explicit independence constraints on the representations is important for estimating ITE from network observational data.
\end{itemize}
\subsection{Hyperparameter Study}
We conduct an analysis of the effects of the most important hyperparameters, $\alpha_1$ and $\alpha_2$, on the performance of TNDVGA. These parameters influence how independence constraints and representation balance contribute to the estimation of ITE from network observational data. The results of the parameter analysis for the BlogCatalog and Flickr datasets, with $\kappa_2$ set to 0.5, 1, and 2, are presented in terms of $\sqrt{\epsilon_{PEHE}}$ and $\epsilon _{ATE}$. We vary $\alpha_1$ and $\alpha_2$ within the range $\{0.01, 0.1, 1, 10, 100\}$. The results of the hyperparameter study are shown in Figs. \ref{fig:hyperparameter study-Blogcatalog} and \ref{fig:hyperparameter study-Flickr}. When $\alpha_1$ and $\alpha_2$ range in  $\{0.01, 0.1, 1\}$, the variations in $\sqrt{\epsilon_{PEHE}}$ and $\epsilon _{ATE}$ are minimal, suggesting that TNDVGA demonstrates stable and favorable performance across a wide range of parameter values.
However, when $\alpha_1\geq10$ or $\alpha_2\geq10$, TNDVGA’s performance in estimating $\epsilon_{ATE}$ noticeably declines. This reduction in performance occurs because the objective function places too much emphasis on the regularization term at high parameter settings, thereby affecting the accuracy of ATE estimation. 

\section{conclusion and future work}\label{sec:conclusion and future work}
This paper aims to improve the accuracy of individual treatment effect estimation from networked observational data by modeling disentangled latent factors. The proposed model, TNDVGA, leverages observed features and auxiliary network information to infer and disentangle four distinct sets of latent factors: instrumental, confounding, adjustment, and noise factors. Empirical results from extensive experiments on two semi-synthetic and one synthetic networked datasets demonstrate that TNDVGA outperforms existing state-of-the-art methods in estimating ITE from networked observational data.


Several promising directions for future work are worth exploring. First, we aim to extend TNDVGA to estimate treatment effects for multiple or continuous treatments, thereby broadening its applicability to a wider range of real-world scenarios. Second, we plan to investigate ITE estimation under network interference within a generative model framework that employs variational inference. Additionally, we intend to address the potential issue of posterior collapse in VAE, which may arise in the current model, and explore strategies to mitigate this challenge.

\bibliographystyle{ACM-Reference-Format}
\bibliography{sample-base}










\end{document}